\documentclass[10pt,letterpaper]{article}
\usepackage[T1]{fontenc}
\usepackage[utf8]{inputenc}
\usepackage[margin=1in,headheight=22pt,headsep=18pt,footskip=28pt]{geometry}
\usepackage{newtxtext}
\usepackage{amsmath,amssymb,amsthm}
\usepackage{newtxmath}
\usepackage{amsmath,amsfonts,bm}

\def\eqref#1{equation~\ref{#1}}
\def\1{\bm{1}}

\DeclareMathAlphabet{\mathsfit}{\encodingdefault}{\sfdefault}{m}{sl}
\SetMathAlphabet{\mathsfit}{bold}{\encodingdefault}{\sfdefault}{bx}{n}

\usepackage[final]{microtype}
\usepackage[scaled=0.92]{helvet}
\usepackage[table]{xcolor}
\usepackage{graphicx,booktabs,multirow,array}
\usepackage{placeins}
\usepackage{algorithm}
\usepackage[noend]{algpseudocode}
\usepackage[round,authoryear]{natbib}
\usepackage{url}
\usepackage{etoolbox}
\usepackage{fontawesome5}
\usepackage{titlesec}
\usepackage{fancyhdr}
\usepackage{caption}
\usepackage{pgfplots}
\pgfplotsset{compat=1.17}
\usepgfplotslibrary{fillbetween}
\usepackage{tikz}
\usepackage[skins,breakable]{tcolorbox}

\definecolor{goldline}{HTML}{987225}
\definecolor{goldrow}{HTML}{FAF3DB}
\definecolor{panelrow}{HTML}{EEEEEA}
\definecolor{lossred}{HTML}{B45D50}
\definecolor{gaingreen}{HTML}{257254}
\definecolor{flatgray}{HTML}{74766F}
\definecolor{inkblue}{HTML}{31596A}
\definecolor{gmInk}{HTML}{222C30}
\definecolor{gmPaper}{HTML}{FAF7EE}
\definecolor{gmRule}{HTML}{D9CFB5}
\usepackage[colorlinks=true,linkcolor=inkblue,citecolor=gaingreen,urlcolor=inkblue,
  pdfborder={0 0 0},bookmarksnumbered=true]{hyperref}
\hypersetup{pdftitle={Fine-Tuning Diffusion Language Models with Context Selection and Target Weighting},
  pdfsubject={GoldiMask},pdfkeywords={diffusion language models, supervised fine-tuning, submodular optimization}}
\newcommand{\GoldiMaskAuthors}{%
  \mbox{Loay Mualem\textsuperscript{1,2}}\quad
  \mbox{Llu\'{i}s Pastor-P\'{e}rez\textsuperscript{2}}\quad
  \mbox{Vinh Tong\textsuperscript{2}}\quad
  \mbox{Andrei Manolache\textsuperscript{1,2,3}}\\[4pt]
  \mbox{Tanja Bien\textsuperscript{2}}\quad
  \mbox{Steffen Staab\textsuperscript{2}}\quad
  \mbox{Mathias Niepert\textsuperscript{1,2}}}
\hypersetup{pdfauthor={Loay Mualem, Lluís Pastor-Pérez, Vinh Tong, Andrei Manolache, Tanja Bien, Steffen Staab, Mathias Niepert}}
\newcommand{\GoldiMaskAffiliations}{%
  \textsuperscript{1}IMPRS-IS\quad
  \textsuperscript{2}Institute for AI, University of Stuttgart\quad
  \textsuperscript{3}Bitdefender}
\newcommand{\GoldiMaskContact}{loay.mualem@ki.uni-stuttgart.de}
\newcommand{\GoldiMaskCodeURL}{}
\newcommand{\GoldiMaskDate}{September 26, 2026}
\newcommand{\GoldiMaskTitle}{Fine-Tuning Diffusion Language Models\\
with Context Selection and Target Weighting}

\titleformat{\section}{\large\bfseries\color{gmInk}}{\thesection}{0.65em}{}
\titleformat{\subsection}{\normalsize\bfseries\color{gmInk}}{\thesubsection}{0.65em}{}
\titleformat{\subsubsection}{\normalsize\bfseries}{\thesubsubsection}{0.65em}{}
\titleformat{\paragraph}[runin]{\normalsize\bfseries}{}{0pt}{}
\titlespacing*{\section}{0pt}{17pt plus 3pt minus 2pt}{7pt}
\titlespacing*{\subsection}{0pt}{12pt plus 2pt minus 1pt}{5pt}
\titlespacing*{\subsubsection}{0pt}{10pt}{4pt}
\titlespacing*{\paragraph}{0pt}{7pt}{0.6em}
\AtBeginEnvironment{table}{\ifcsdef{gmResultSD}{\let\gmResultSD\gmArxivResultSD}{}}
\newcommand{\gmRunningWordmark}{%
  \begin{tikzpicture}[baseline=-0.7pt,x=0.46pt,y=0.46pt]
    \foreach \x/\y in {0/14,7/14,14/14,0/7,7/7,14/7,0/0,7/0,14/0}
      \draw[draw=goldline!45,line width=0.25pt,rounded corners=0.4pt,fill=white]
        (\x,\y) rectangle ++(5,5);
    \foreach \x/\y in {0/14,7/7,14/0}
      \fill[goldline,rounded corners=0.4pt] (\x,\y) rectangle ++(5,5);
    \fill[goldline!35,rounded corners=0.4pt] (14,14) rectangle ++(5,5);
  \end{tikzpicture}\hspace{4pt}%
  {\sffamily\bfseries\small\textcolor{goldline}{Goldi}\textcolor{gmInk}{Mask}}}
\renewcommand{\headrulewidth}{0.35pt}
\renewcommand{\headrule}{\hbox to\headwidth{\color{gmRule}\leaders\hrule height \headrulewidth\hfill}}
\fancypagestyle{plain}{\fancyhf{}\fancyfoot[C]{\sffamily\small\textcolor{flatgray}{\thepage}}\renewcommand{\headrulewidth}{0pt}}

\newtheorem{proposition}{Proposition}

\newtheorem{corollary}{Corollary}

\newtcolorbox{gmtakeaway}[1]{enhanced,breakable,
  colback=gmPaper,colframe=gmRule,boxrule=0.4pt,
  borderline west={2pt}{0pt}{goldline},arc=2pt,
  left=9pt,right=9pt,top=6pt,bottom=6pt,
  before skip=10pt,after skip=10pt,
  before upper={\textcolor{goldline}{\textbf{Takeaway~#1.}}\enspace}}
\newcommand{\loaynote}[1]{}
\newcommand{\mn}[1]{}
\newcommand{\vinh}[1]{}
\newcommand{\lpp}[1]{}

\begin{document}
\thispagestyle{plain}
\begin{tcolorbox}[enhanced,colback=gmPaper,colframe=gmRule,
  boxrule=0.45pt,arc=3pt,left=16pt,right=16pt,top=17pt,bottom=14pt,
  borderline north={2pt}{0pt}{goldline},before skip=8pt,after skip=18pt]
\begin{center}
{\sffamily\bfseries\fontsize{23}{26}\selectfont
\begin{tikzpicture}[baseline=-1.3pt,x=1pt,y=1pt]
  % A compact token grid: visible context in gold, masked positions in outline.
  \foreach \x/\y in {0/14,7/14,14/14,0/7,7/7,14/7,0/0,7/0,14/0}
    \draw[draw=goldline!45,line width=0.45pt,rounded corners=1pt,fill=gmPaper]
      (\x,\y) rectangle ++(5,5);
  \foreach \x/\y in {0/14,7/7,14/0}
    \fill[goldline,rounded corners=1pt] (\x,\y) rectangle ++(5,5);
  \fill[goldline!35,rounded corners=1pt] (14,14) rectangle ++(5,5);
\end{tikzpicture}\hspace{7pt}%
\textcolor{goldline}{Goldi}\textcolor{gmInk}{Mask}\par}
\vspace{7pt}
{\bfseries\fontsize{19}{23}\selectfont\textcolor{gmInk}{\GoldiMaskTitle}\par}
\vspace{10pt}
\ifdefempty{\GoldiMaskAuthors}{}{%
  {\normalsize\bfseries\GoldiMaskAuthors\par}%
  \ifdefempty{\GoldiMaskAffiliations}{}{\vspace{5pt}{\small\GoldiMaskAffiliations\par}}%
}
\end{center}
\vspace{5pt}
{\sffamily\small\bfseries\textcolor{goldline}{ABSTRACT}}\par
Supervised fine-tuning of discrete diffusion language models masks some response tokens and trains the model to recover their original values from the visible context. The masking pattern therefore determines both the context available to the model and the tokens it learns to predict. Uniform random masking does not explicitly account for the interaction between these choices. We introduce GoldiMask, which selects tokens to reveal as context by approximately maximizing a submodular objective. This objective uses model signals to balance the benefit of revealing tokens against their value as prediction targets. GoldiMask then weights the remaining targets according to how they benefit from the selected context and their remaining learning potential. Across three backbones and three training datasets, GoldiMask achieves the highest average accuracy in most evaluated settings, demonstrating gains on both reasoning and code generation. Component ablations show that both context selection and target weighting contribute to the gains. GoldiMask also reduces decoding iterations on GSM8K and MATH-500 under confidence-threshold parallel decoding, while maintaining comparable accuracy at higher confidence thresholds. Project page with interactive visualizations: \url{https://loaym.github.io/GoldiMask/}.
\par\vspace{8pt}
{\color{gmRule}\hrule height 0.4pt}
\vspace{7pt}
{\footnotesize\textbf{Keywords}\enspace Diffusion language models\enspace\textcolor{goldline}{\textbullet}\enspace
Supervised fine-tuning\enspace\textcolor{goldline}{\textbullet}\enspace Submodular optimization\par}
\par\vspace{5pt}
{\small
\renewcommand{\arraystretch}{1.25}
\begin{tabular}{@{}c@{\hspace{6pt}}l@{}}
\textcolor{goldline}{\faIcon{calendar-alt}} & \textbf{Date:} \GoldiMaskDate \\
\textcolor{goldline}{\faGithub} & \textbf{Code:} \ifdefempty{\GoldiMaskCodeURL}{Available soon.}{\url{\GoldiMaskCodeURL}} \\
\textcolor{goldline}{\faGlobe} & \textbf{Project page:} \url{https://loaym.github.io/GoldiMask/} \\
\textcolor{inkblue}{\faEnvelope} & \textbf{Contact:} \href{mailto:\GoldiMaskContact}{\GoldiMaskContact}
\end{tabular}\par}
\end{tcolorbox}

\section{Introduction}

Diffusion Large Language Models (DLLMs) have been recently proposed as an alternative to Autoregressive Large Language Models (ARLMs). ARLMs predict tokens one after the other, while DLLMs can predict multiple tokens in parallel, resulting in two advantages; parallel decoding~\citep{chen2026dflash,wu2026fastdllm,chen2026dparallel}, and revoking earlier decisions, which is especially valuable in domains like robotics \citep{liang2026discrete-robot}.

% Masked diffusion large language models generate text by repeatedly predicting missing tokens from partial context, allowing multiple tokens to be decoded in parallel. Supervised fine-tuning (SFT) trains this ability by masking part of a response and reconstructing it from the visible tokens. The mask therefore defines the learning problem: it determines both the evidence available to the model and the targets it must predict. Choosing useful context and choosing useful supervision are coupled decisions.

As with ARLMs, recent works have explored post-training in DLLMs \citep{nie2025large-posttraining,ye2025dream7bdiffusionlarge} in order to improve their performance and their reasoning capabilities. During supervised fine-tuning (SFT), some tokens are masked independently, and the model learns to reconstruct the full text. This masking process is generally uniformly at random. The mask therefore defines the learning problem: it determines both the evidence available to the model and the targets it must predict. Choosing useful context and choosing useful supervision are coupled decisions. A random split pairs the two blindly, which makes SFT in DLLMs more challenging than in ARLMs, and a poorly chosen masking scheme can even reduce the base model's accuracy~\citep{parashar2026learnabilityinformed-lift,piskorz2025masks}.

Recent methods adapt different aspects of diffusion SFT, including the masking rate~\citep{xu2026gift}, the choice of revealed tokens~\citep{parashar2026learnabilityinformed-lift}, and the weighting of prediction losses~\citep{ye2025dream7bdiffusionlarge,deng2026beyond-agdo}. These methods use signals such as entropy, confidence, position, or attention to guide training. However, revealing a token changes both the context available to other targets and which tokens remain supervised, motivating a joint treatment of context selection and target weighting.

% Uniform random masking does not account for this coupling. It can hide evidence needed by the remaining targets or spend supervision on tokens already clear from context. Recent approaches adapt the masking rate, the revealed positions, or the loss weights (Xu et al., 2026; Parashar et al., 2026a; Ye et al., 2025; Deng et al., 2026). We bring reveal selection and loss weighting together around a simple principle: provide context that helps valuable targets, then emphasize the targets that benefit from that context and still have room to improve.

To address this coupling, we introduce GoldiMask, an SFT strategy that first selects which tokens to reveal and then determines how strongly to supervise the remaining targets. At a sampled masking ratio, GoldiMask scores a reveal set by the support it provides to the tokens left masked, weighted by their supervision priority. This priority favors intermediate gold-token probabilities (probabilities of correct tokens), balancing local probability responsiveness with remaining prediction headroom. Support is estimated from attention to the revealed positions, with diminishing returns as more context is supplied.
% Revealing a token also removes it from the supervised targets, so the objective accounts for both the context it provides and the supervision it gives up.
This selection is interaction-aware, since the value of revealing a token depends on the context already selected and the supervision it gives up.
We formulate this decision as a non-monotone submodular set-selection problem.

GoldiMask then evaluates the remaining targets with the selected context. It assigns larger loss weights to tokens whose gold-token probability increases and whose updated confidence still indicates useful supervision. Capping and normalizing these weights limits concentration while preserving their total mass. Together, selection and weighting allocate context and supervision within each training example.

We evaluate GoldiMask on reasoning and code-generation tasks across three backbones and three training datasets. GoldiMask achieves the highest average accuracy in most evaluated settings, with ablations supporting the contributions of adaptive context selection and target weighting.

GoldiMask's focus on learning from partial context also motivates evaluating parallel decoding. Predicting several tokens in one iteration requires the model to use the currently available context without waiting for the other predictions to be revealed. We therefore investigate whether coordinating context and supervision during training enables more tokens to be decoded per iteration while maintaining accuracy. On GSM8K and MATH-500, GoldiMask reduces decoding iterations, often maintaining comparable accuracy (Section~\ref{sec:parallel-decoding}).

% This approach is particularly relevant to parallel decoding, where several unresolved tokens must be predicted from the same partial context. Across evaluated settings spanning three backbones and two SFT datasets, GoldiMask improves average benchmark accuracy over the compared baselines, with ablations showing contributions from both selection and weighting. Our parallel-decoding evaluations further demonstrate that the benefits extend to decoding behavior, connecting the design of SFT supervision to the model's ability to generate tokens in parallel.

% \vinh{The contributions need to be rewritten. You have 3 bullet points, (i) in the first point you only talk about the \textit{learnability curve} - how about support curve - which is also super-important. I would suggest frame the first contribution as a \textit{teaching-value curve + support curve}. This better matches the actual method terminologies. (ii) the second point is framed mainly as corruption design as set selection, which weighting mentioned afterward. Since weigthing is empirically crucial, so it should have equal status with selection rather than secondary. I would replace this second bullet point with a two-stage allocation rule - reveal selection + realized gain weighting. (iii) I would keep but add more about parallel decoding stuff. }
%
\paragraph{Our contributions.}
\begin{itemize}
\item We formulate masking in diffusion SFT as a coupled subset-selection problem and introduce GoldiMask, combining non-monotone submodular reveal selection with response-based, budget-preserving loss weighting.
\item We evaluate GoldiMask across three backbones and three training datasets, achieving the highest average accuracy in most evaluated reasoning and code-generation settings. Ablations support the contributions of adaptive selection and target weighting and examine the choice of supervision-priority function.
\item We evaluate parallel decoding across three backbones. Under confidence-threshold decoding, GoldiMask reduces decoding iterations on GSM8K and MATH-500, often maintaining comparable accuracy.
% The theoretical analysis and quantitative parallel-decoding results will be
% supplied in the reserved subsections; retain this claim for their integration.
\end{itemize}

\section{Related Work}
%\loaynote{Loay is reviewing this section in the end. Lluis/Vinh please double check things and probably cut things if any of you see this message}
\label{sec:related}

In what follows, we review diffusion language models and recent approaches to their post-training. We refer the reader to Appendix~\ref{app:related} for additional related work.

\paragraph{Diffusion language models.} Motivated by their success in continuous domains such as image generation \citep{diffusion1,dhariwal2021diffusion}, diffusion language models have appeared as an alternative paradigm to autoregressive language modeling. Diffusion language models operate either in continuous token-embedding spaces~\citep{hu2026elf,chen2026langflow} or directly over discrete tokens. Discrete diffusion language models have already shown competitive performance with similarly sized autoregressive models on tasks such as code generation and mathematical reasoning \citep{nie2025large-posttraining,zhu2025llada15,ye2025dream7bdiffusionlarge}. In this work, we focus on masked discrete diffusion language models, where the choice of which tokens are revealed is crucial for both inference and fine-tuning.

\paragraph{Post-Training.} Vanilla diffusion SFT randomly masks response tokens and trains the model to reconstruct them from the visible context. Recent work adapts this process to make supervision more effective. LIFT~\citep{parashar2026learnabilityinformed-lift} selects targets by ground-truth confidence rank using a noise-dependent schedule; GIFT~\citep{xu2026gift} assigns token-specific masking rates based on entropy; CART~\citep{ye2025dream7bdiffusionlarge} weights targets by proximity to visible context; and AGDO~\citep{deng2026beyond-agdo} uses attention to guide denoising and loss weighting. GoldiMask explicitly accounts for the coupling between context and supervision: revealing a token supports other targets but removes that token from supervision. We capture this trade-off in a submodular set objective, then weight the remaining targets by their measured response to the selected context and their supervision priority.

% Suggested methodology, preserving and consolidating Mathias's draft in
% sections/method_mathias.tex. Include this file instead of that source.
% AUTHOR CONFIRMED: piecewise rho sampling and uniform supervision when |C| <= K;
% residual weight is distributed uniformly when positive utilities cannot absorb it.
% AUTHOR CHECK: confirm the empty-supervision convention against the training code.

\section{Methodology}
\label{sec:method-new}

In this section, we present GoldiMask, a two-pass procedure for choosing
which response tokens to reveal and how strongly to supervise those that
remain masked. We first define the training setup and objective, then
describe reveal-set selection and target weighting.
Figure~\ref{fig:main-fig} illustrates the procedure. Algorithm~\ref{alg:goldimask}
in Appendix~\ref{app:main-algorithm} gives the pseudocode for one training step.

\begin{figure}[!htbp]
\centering
\includegraphics[width=\linewidth]{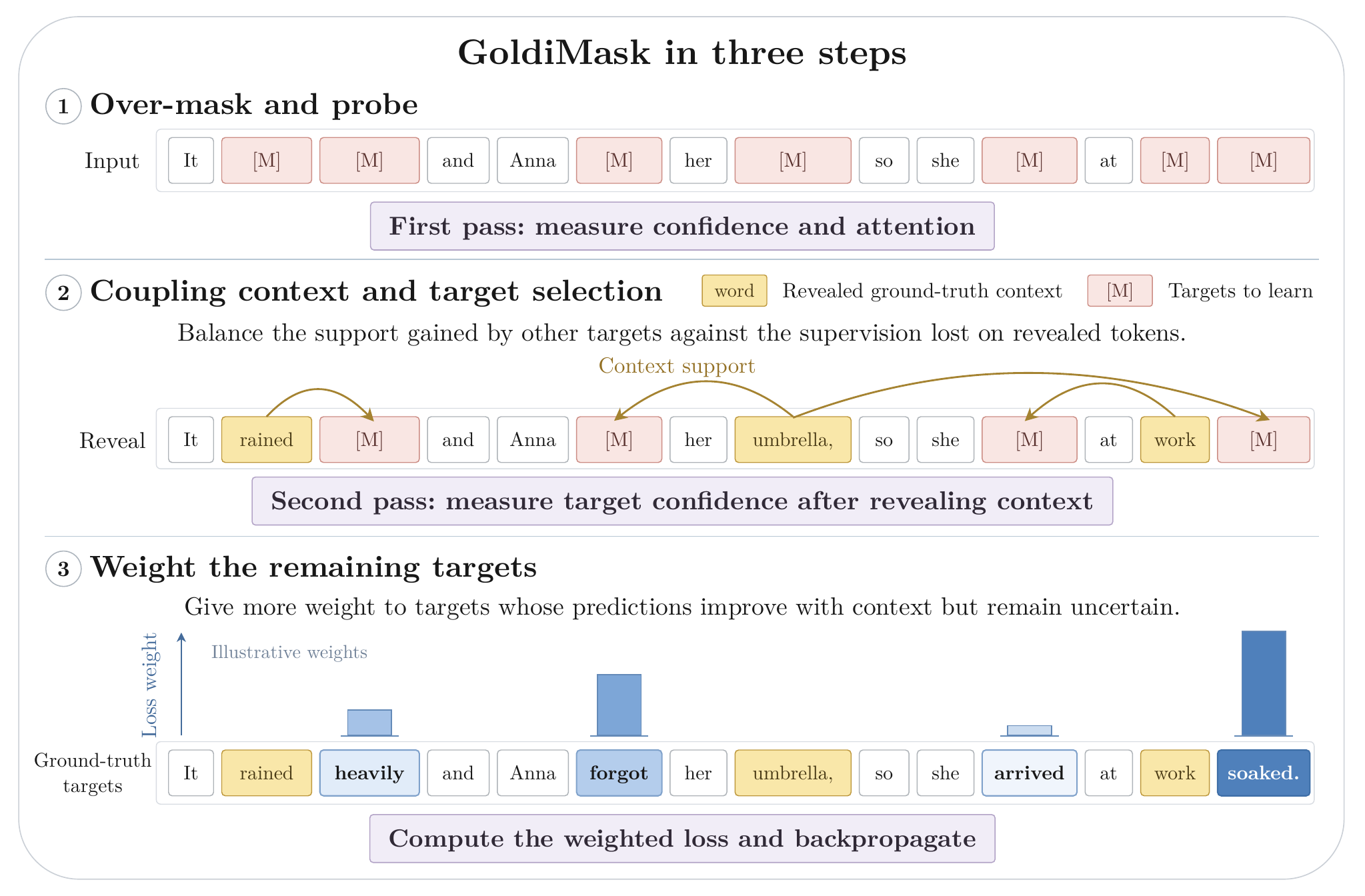}%{goldimask-main.drawio.pdf}
\caption{Overview of GoldiMask. A first forward pass on the over-masked
response provides ground-truth token probabilities and attention scores for
selecting a reveal set. A second forward pass evaluates the remaining
targets with the selected context. Their positive context gains and
post-reveal priorities determine loss weights through water-filling for the
training update. The supervision-priority function $\lambda$ and support
response $\phi$ remain fixed throughout fine-tuning.}
\label{fig:main-fig}
\end{figure}

\subsection{Setup and training objective}
\label{sec:setup}

In supervised fine-tuning, each example is a prompt--response pair
$x_0=(q,y)$, where $q$ denotes the prompt and $y$ the ground-truth response.
The prompt remains visible, and masking applies only to the response.
Let $[L]=\{1,\ldots,L\}$ index the $L$ supervisable response positions,
excluding padding and special tokens, and let $y^i$ denote the token at
position $i$. The noise level $t\sim\mathcal U(0,1)$ specifies the desired
final masking fraction.

We initially mask additional response tokens and then reveal a selected
subset~\citep{parashar2026learnabilityinformed-lift}. We draw
$\rho\sim\mathcal U(0.1,1-t)$ for $t<0.9$ and
$\rho\sim\mathcal U(0,1-t)$ otherwise, applying the same procedure in both
regimes. We independently replace each response token
with \texttt{[MASK]} with probability $t+\rho$, producing a partially masked
response $\widetilde y$ and input $\widetilde x=(q,\widetilde y)$.
Let $\mathcal C\subseteq[L]$ be the set of positions masked in this step.
The desired supervised count and corresponding reveal budget are
\begin{equation}
K=\lfloor tL\rfloor,\qquad B=|\mathcal C|-K.
\label{eq:gm-budgets}
\end{equation}
If $|\mathcal C|\leq K$, we reveal no tokens and apply uniformly weighted
supervision to all masked candidates. For $1\leq K<|\mathcal C|$, we select a subset
$R\subseteq\mathcal C$ of size $B$, called the \emph{reveal set}, and
replace each mask at position $i\in R$ with its ground-truth token $y^i$.
Denote the response after these replacements by $y_t^R$ and the
corresponding model input by $x_t^R=(q,y_t^R)$.
The positions $\mathcal S=\mathcal C\setminus R$ remain masked, and their
original tokens are the \emph{supervised targets}. All visible tokens,
including the prompt, the initially unmasked response tokens, and the
revealed tokens, form the \emph{context}.

Given the input $x_t^R$, we train the model to predict the ground-truth
tokens at the masked positions $\mathcal S$. Let $\theta$ denote the
model parameters and $p_\theta(y^i\mid x_t^R)$ the probability assigned
to the ground-truth token $y^i$. GoldiMask assigns each supervised position
a nonnegative weight $w_i$ that controls its contribution to the per-example
loss:
\begin{equation}
\ell_{\mathrm{GM}}(\theta)
=-\frac{1}{tL}\sum_{i\in\mathcal S}
w_i\log p_\theta(y^i\mid x_t^R).
\label{eq:loss}
\end{equation}
Standard masked SFT uses uniform token weights, $w_i=1$.
GoldiMask uses adaptive weights, normalized so that
\begin{equation}
\sum_{i\in\mathcal S}w_i=|\mathcal S|.
\label{eq:weight-budget}
\end{equation}
Thus, the average weight across supervised positions remains one.
The reveal set $R$ and the weights $w_i$ are held fixed during
differentiation.

\subsection{Selecting useful context}
\label{sec:selection}
\label{sec:objective}

Selecting $R$ balances the context supplied to other tokens against the
value of retaining the revealed tokens as prediction targets.
The first forward pass on $\widetilde x$, without gradient tracking,
provides ground-truth token probabilities
$p_i^{(1)}=p_\theta(y^i\mid\widetilde x)$ for $i\in\mathcal C$ and
attention scores $A_{ij}$ from position $i$ to $j$, averaged over heads
in the final attention block. We set $A_{ii}=0$ and hold these quantities
fixed during selection.

For each remaining target $i\in\mathcal C\setminus R$, its total attention
to the reveal set serves as a proxy for contextual support:
\begin{equation}
S_i(R)=\sum_{j\in R}A_{ij}.
\label{eq:support}
\end{equation}
We select $R$ to maximize the priority-weighted support of the remaining
targets:
\begin{equation}
F(R)=\sum_{i\in\mathcal C\setminus R}
\lambda\!\left(p_i^{(1)}\right)\phi\!\left(S_i(R)\right),
\qquad
R^\star\in\arg\max_{R\subseteq\mathcal C:\,|R|=B}F(R).
\label{eq:objective}
\end{equation}
Here, $\lambda(p_i^{(1)})$ scores the priority of retaining token $i$ as a
prediction target, while $\phi(S_i(R))$ estimates the benefit of its
contextual support. Revealing a token can increase support for other
targets but removes its own contribution from $F$.

For a ground-truth token probability $p\in[0,1]$, we use the supervision
priority \mbox{$\lambda(p)=p(1-p)^3$}, which peaks at $p=1/4$.
It combines the approximate probability
increase $p(1-p)^2$ under an idealized cross-entropy update with the remaining
probability gap $1-p$, favoring targets with room to improve
(Appendix~\ref{app:gm-lambda-theory}). Studies of supervision allocation
and context response appear in Appendices~\ref{app:gm-training-study}
and~\ref{app:gm-lambda-evidence}, respectively.

For a nonnegative attention-support value $s$, we use
$\phi(s)=s/(s+s_0)$, with $s_0=0.05$.
This increasing, saturating response assigns less value to additional
support for targets that already receive substantial context.
Appendices~\ref{app:gm-support-theory} and~\ref{app:gm-support} give its
diminishing-returns properties and empirical support fits. Both functions
remain fixed throughout training; their probability and attention inputs
are recomputed for each example.

With fixed nonnegative priorities and attention weights, the nondecreasing,
concave $\phi$ makes $F$ submodular but generally non-monotone.
The proof appears in Appendix~\ref{app:gm-selection-properties}.
We evaluate four submodular maximization algorithms, including practical
adaptations of established methods. The main results report deterministic
greedy (GoldiMask-D) and randomized greedy (GoldiMask-R)~\citep{buchbinder2014submodular}.
Appendix~\ref{app:arms} describes all four selectors and their implementation
adaptations.

\subsection{Weighting the remaining targets}
\label{sec:weighting}

After selecting $R$, a second forward pass on $x_t^R$ provides updated
ground-truth token probabilities
$p_i^{(2)}=p_\theta(y^i\mid x_t^R)$ for $i\in\mathcal S$.
We reuse $\lambda$ to assess the remaining targets' priority under this
context and combine it with their positive log-probability gains:
\begin{equation}
u_i=\big[\log p_i^{(2)}-\log p_i^{(1)}\big]_+
\lambda\!\left(p_i^{(2)}\right),
\qquad [a]_+=\max(a,0).
\label{eq:utility}
\end{equation}
The first factor identifies targets helped by the reveal. The second
favors targets that still merit supervision after receiving that help.
Selection evaluates priority before the context has been chosen, whereas
weighting evaluates it under the context used for training. Both
probabilities are computed with the same model parameters, so the
log-probability difference measures the response to the reveal before
the training update.

We allocate the total weight $n=|\mathcal S|$ in proportion to $u_i$,
subject to a per-token cap $c=10$. Whenever the positive utilities can
absorb the budget, water-filling gives
\begin{equation}
w_i=\min\!\left(\frac{u_i}{\nu},c\right),\qquad
\sum_{i\in\mathcal S}\min\!\left(\frac{u_i}{\nu},c\right)=n.
\label{eq:waterfill}
\end{equation}
The water-filling level $\nu$ is found by bisection in the nondegenerate case.
The cap limits weight concentration while water-filling preserves the
total loss-weight mass. Appendix~\ref{app:gm-weight-properties} gives
the allocation's characterization and boundary cases.
The weights are detached and act as fixed coefficients of the
cross-entropy loss in the backward pass. Each GoldiMask intervention thus
uses two forward passes and one backward pass.

\section{Results and Analysis}
\label{sec:experiments}
\label{sec:results}
\newcommand{\gmResultSD}[2]{#1{\scriptsize\color{black!65}\,$\pm$\,#2}}
In this section, we evaluate our approach through the following research questions.
\begin{itemize}
\item[\textbf{RQ1.}] Does GoldiMask improve reasoning accuracy over existing
supervised fine-tuning methods? (Section~\ref{sec:main-results})
\item[\textbf{RQ2.}] Do these benefits extend to code generation?
(Section~\ref{sec:code-results})
\item[\textbf{RQ3.}] How does GoldiMask affect the accuracy--decoding-step
trade-off? (Section~\ref{sec:parallel-decoding})
\item[\textbf{RQ4.}] How do context selection, target weighting, and their
design choices affect performance? (Section~\ref{sec:component-results})
% \item[\textbf{RQ5.}] Can GoldiMask retain its benefits with a single training
% forward pass? (Section~\ref{sec:single-pass-results})
\end{itemize}

\subsection{Experimental Setup}
\label{sec:experimental-setup}

\paragraph{Models and training data.}
We use LLaDA-8B-Instruct~\citep{nie2025large-posttraining},
LLaDA-1.5~\citep{zhu2025llada15}, and Dream-7B~\citep{ye2025dream7bdiffusionlarge}.
We fine-tune all three on s1K~\citep{muennighoff2025s1},
which contains 1,000 chain-of-thought examples. We additionally fine-tune
LLaDA-8B-Instruct on LIFT-SFT-12K~\citep{parashar2026learnabilityinformed-lift}
as a second reasoning corpus. For code generation, we fine-tune
LLaDA-8B-Instruct and Dream-7B on KodCode~\citep{xu2025kodcode}
(Section~\ref{sec:code-results}).

\paragraph{Baselines.}
We compare against the base models, Vanilla SFT~\citep{sahoo2024simple},
and recent adaptive SFT methods including CART~\citep{ye2025dream7bdiffusionlarge},
GIFT~\citep{xu2026gift}, LIFT~\citep{parashar2026learnabilityinformed-lift},
and AGDO~\citep{deng2026beyond-agdo}. These adaptive methods modify masking
or loss weighting (Section~\ref{sec:related}).
All reported fine-tuned baselines are retrained in our harness with matched
hyperparameters; each results table lists the methods evaluated in that setting.

\paragraph{Training and evaluation.}
We follow the protocol of \citet{parashar2026learnabilityinformed-lift}, including
their LoRA and optimizer settings. On s1K, training uses a LoRA rank of $r=128$,
a global batch size of $8$, and $2{,}480$ steps. We evaluate pass@1 accuracy on
GSM8K, MATH-500, Countdown, and Sudoku across three seeds.
We use greedy decoding and a common checkpoint and generation-length selection
protocol across compared methods (Appendix~\ref{app:checkpoint-selection}).
Fixed-length results are reported in Appendix~\ref{app:fixed-task-lengths}.
On LLaDA-8B-Instruct, GoldiMask-D/R require approximately $1.2\times$ Vanilla
SFT's training time, comparable to LIFT (Appendix~\ref{sec:computeresources}).

\subsection{Reasoning Performance}
\label{sec:main-results}

We first compare GoldiMask's reasoning accuracy with existing supervised
fine-tuning methods across the tested backbones and training corpora.
On s1K, both GoldiMask-D and GoldiMask-R exceed
the strongest SFT baseline in average accuracy across the four tasks on every backbone
(Table~\ref{tab:main}). The strongest variant varies across models, but the
average improvement is shared by the two maximizers.

% Validation-selected results updated from author-approved rows on 2026-09-23.
% AGDO means and sample SDs recomputed from three seeds; see output/agdo_seed_update_20260923.
\begin{table}[!t]
\centering
\setlength{\tabcolsep}{5.5pt}
\hypersetup{hidelinks}
\caption{\textbf{Main results.} Fine-tuning on s1K across three diffusion language models. Values are means with standard deviations in smaller $\pm$ entries. Baselines are retrained in our harness with matched hyperparameters. $^{\ast}$LIFT is the stronger of its $H{=}2$ and $H{=}3$ variants per setting, by four-benchmark average. Relative changes are against the strongest baseline in each block. $\dagger$ denotes our code implementation; `w' denotes a workshop.}
\label{tab:main}
\small
\begin{tabular}{l>{\scriptsize\color{black!65}}l|cccc>{\columncolor{panelrow}}c}
\toprule
\textbf{Method} & \multicolumn{1}{l|}{\textbf{Venue}} & \textbf{GSM8K} & \textbf{MATH} & \textbf{Countdown} & \textbf{Sudoku} & \textbf{Average} \\
\midrule

\multicolumn{7}{c}{\textbf{LLaDA-8B-Instruct}}\\
\midrule
LLaDA-8B & -- & 78.1 & 36.1 & 19.6 & 11.2 & 36.3 \\
Vanilla {\scriptsize\color{black!65}\citep{sahoo2024simple}} & NeurIPS'24 & \gmResultSD{79.2}{0.2} & \gmResultSD{36.2}{0.9} & \gmResultSD{17.2}{2.1} & \gmResultSD{14.5}{1.3} & \gmResultSD{36.8}{0.3} \\
GIFT {\scriptsize\color{black!65}\citep{xu2026gift}} & ICLRw'26 & \gmResultSD{\textbf{81.0}}{0.1} & \gmResultSD{37.3}{0.4} & \gmResultSD{21.5}{2.8} & \gmResultSD{15.8}{0.4} & \gmResultSD{38.9}{0.5} \\
CART {\scriptsize\color{black!65}\citep{ye2025dream7bdiffusionlarge}} & arXiv'25 & \gmResultSD{80.4}{1.0} & \gmResultSD{35.9}{1.6} & \gmResultSD{24.0}{1.4} & \gmResultSD{14.4}{1.6} & \gmResultSD{38.7}{0.6} \\
AGDO$^\dagger$ {\scriptsize\color{black!65}\citep{deng2026beyond-agdo}} & ACL'26 & \gmResultSD{78.3}{0.8} & \gmResultSD{37.7}{0.1} & \gmResultSD{16.9}{1.3} & \gmResultSD{12.6}{0.7} & \gmResultSD{36.4}{0.3} \\
LIFT$^{\ast}$ {\scriptsize\color{black!65}\citep{parashar2026learnabilityinformed-lift}} & ICML'26 & \gmResultSD{79.1}{0.8} & \gmResultSD{36.6}{1.0} & \gmResultSD{28.8}{2.6} & \gmResultSD{13.5}{1.0} & \gmResultSD{39.5}{0.4} \\

\rowcolor{goldrow}
\textbf{GoldiMask-R} & Ours & \gmResultSD{80.4}{0.2} & \gmResultSD{38.3}{1.0} & \gmResultSD{\textbf{29.7}}{0.0} & \gmResultSD{14.4}{1.5} & \gmResultSD{40.7}{0.6} \\
\rowcolor{goldrow}
\multicolumn{2}{l}{{\scriptsize\color{gray}vs.\ LIFT}} & {\color{gaingreen}$\uparrow$\,\textbf{1.6\%}} & {\color{gaingreen}$\uparrow$\,\textbf{4.6\%}} & {\color{gaingreen}$\uparrow$\,\textbf{3.1\%}} & {\color{gaingreen}$\uparrow$\,\textbf{6.7\%}} & {\color{gaingreen}$\uparrow$\,\textbf{3.0\%}} \\

\rowcolor{goldrow}
\textbf{GoldiMask-D} & Ours & \gmResultSD{80.4}{1.2} & \gmResultSD{\textbf{39.5}}{0.8} & \gmResultSD{27.7}{2.7} & \gmResultSD{\textbf{16.9}}{0.5} & \gmResultSD{\textbf{41.1}}{0.7} \\
\rowcolor{goldrow}
\multicolumn{2}{l}{{\scriptsize\color{gray}vs.\ LIFT}} & {\color{gaingreen}$\uparrow$\,\textbf{1.6\%}} & {\color{gaingreen}$\uparrow$\,\textbf{7.9\%}} & {\color{lossred}$\downarrow$\,\textbf{3.8\%}} & {\color{gaingreen}$\uparrow$\,\textbf{25.2\%}} & {\color{gaingreen}$\uparrow$\,\textbf{4.1\%}} \\

\midrule
\multicolumn{7}{c}{\textbf{LLaDA-1.5}}\\
\midrule
LLaDA-1.5 & -- & 80.9 & 37.8 & 22.6 & 12.1 & 38.4 \\
Vanilla & NeurIPS'24 & \gmResultSD{79.2}{0.5} & \gmResultSD{32.6}{1.6} & \gmResultSD{22.0}{1.0} & \gmResultSD{14.4}{1.6} & \gmResultSD{37.1}{0.9} \\
GIFT & ICLRw'26 & \gmResultSD{80.4}{1.3} & \gmResultSD{35.9}{0.1} & \gmResultSD{19.7}{0.8} & \gmResultSD{19.0}{0.4} & \gmResultSD{38.8}{0.5} \\
CART & arXiv'25 & \gmResultSD{81.4}{0.8} & \gmResultSD{37.9}{1.3} & \gmResultSD{22.4}{2.2} & \gmResultSD{17.7}{0.9} & \gmResultSD{39.8}{0.1} \\
AGDO$^\dagger$ & ACL'26 & \gmResultSD{80.0}{0.5} & \gmResultSD{39.0}{0.6} & \gmResultSD{21.5}{2.2} & \gmResultSD{14.5}{0.6} & \gmResultSD{38.7}{0.4} \\
LIFT$^{\ast}$ & ICML'26 & \gmResultSD{80.4}{1.0} & \gmResultSD{37.7}{1.6} & \gmResultSD{25.7}{1.5} & \gmResultSD{16.1}{0.8} & \gmResultSD{40.0}{0.4} \\

\rowcolor{goldrow}
\textbf{GoldiMask-R} & Ours & \gmResultSD{81.7}{1.0} & \gmResultSD{\textbf{40.5}}{2.7} & \gmResultSD{\textbf{29.9}}{3.0} & \gmResultSD{17.9}{0.9} & \gmResultSD{\textbf{42.5}}{1.5} \\
\rowcolor{goldrow}
\multicolumn{2}{l}{{\scriptsize\color{gray}vs.\ LIFT}} & {\color{gaingreen}$\uparrow$\,\textbf{1.6\%}} & {\color{gaingreen}$\uparrow$\,\textbf{7.4\%}} & {\color{gaingreen}$\uparrow$\,\textbf{16.3\%}} & {\color{gaingreen}$\uparrow$\,\textbf{11.2\%}} & {\color{gaingreen}$\uparrow$\,\textbf{6.2\%}} \\

\rowcolor{goldrow}
\textbf{GoldiMask-D} & Ours & \gmResultSD{\textbf{81.9}}{0.8} & \gmResultSD{37.6}{0.2} & \gmResultSD{26.6}{1.6} & \gmResultSD{\textbf{21.6}}{1.8} & \gmResultSD{41.9}{0.8} \\
\rowcolor{goldrow}
\multicolumn{2}{l}{{\scriptsize\color{gray}vs.\ LIFT}} & {\color{gaingreen}$\uparrow$\,\textbf{1.9\%}} & {\color{lossred}$\downarrow$\,\textbf{0.3\%}} & {\color{gaingreen}$\uparrow$\,\textbf{3.5\%}} & {\color{gaingreen}$\uparrow$\,\textbf{34.2\%}} & {\color{gaingreen}$\uparrow$\,\textbf{4.7\%}} \\

\midrule
\multicolumn{7}{c}{\textbf{Dream-7B}}\\
\midrule
Dream-7B & -- & 76.7 & 39.8 & 21.1 & 8.2 & 36.5 \\
Vanilla & NeurIPS'24 & \gmResultSD{76.1}{0.4} & \gmResultSD{30.6}{0.3} & \gmResultSD{25.0}{0.9} & \gmResultSD{14.8}{2.1} & \gmResultSD{36.6}{0.5} \\
GIFT & ICLRw'26 & \gmResultSD{78.7}{0.7} & \gmResultSD{40.2}{0.8} & \gmResultSD{24.0}{1.4} & \gmResultSD{22.7}{1.4} & \gmResultSD{41.4}{0.0} \\
CART & arXiv'25 & \gmResultSD{77.8}{0.2} & \gmResultSD{38.9}{0.6} & \gmResultSD{22.3}{0.8} & \gmResultSD{17.2}{2.7} & \gmResultSD{39.1}{0.8} \\
AGDO$^\dagger$ & ACL'26 & \gmResultSD{77.6}{0.4} & \gmResultSD{40.3}{0.8} & \gmResultSD{29.6}{1.2} & \gmResultSD{10.8}{2.1} & \gmResultSD{39.6}{0.9} \\
LIFT$^{\ast}$$^\dagger$ & ICML'26 & \gmResultSD{78.4}{0.1} & \gmResultSD{42.1}{0.7} & \gmResultSD{28.7}{0.8} & \gmResultSD{23.3}{2.5} & \gmResultSD{43.1}{0.7} \\

\rowcolor{goldrow}
\textbf{GoldiMask-R} & Ours & \gmResultSD{\textbf{79.6}}{0.5} & \gmResultSD{\textbf{44.3}}{0.1} & \gmResultSD{33.1}{3.6} & \gmResultSD{\textbf{24.3}}{0.2} & \gmResultSD{\textbf{45.3}}{1.0} \\
\rowcolor{goldrow}
\multicolumn{2}{l}{{\scriptsize\color{gray}vs.\ LIFT}} & {\color{gaingreen}$\uparrow$\,\textbf{1.5\%}} & {\color{gaingreen}$\uparrow$\,\textbf{5.2\%}} & {\color{gaingreen}$\uparrow$\,\textbf{15.3\%}} & {\color{gaingreen}$\uparrow$\,\textbf{4.3\%}} & {\color{gaingreen}$\uparrow$\,\textbf{5.1\%}} \\

\rowcolor{goldrow}
\textbf{GoldiMask-D} & Ours & \gmResultSD{78.8}{0.6} & \gmResultSD{42.8}{0.4} & \gmResultSD{\textbf{37.0}}{4.0} & \gmResultSD{22.8}{2.1} & \gmResultSD{\textbf{45.3}}{1.6} \\
\rowcolor{goldrow}
\multicolumn{2}{l}{{\scriptsize\color{gray}vs.\ LIFT}} & {\color{gaingreen}$\uparrow$\,\textbf{0.5\%}} & {\color{gaingreen}$\uparrow$\,\textbf{1.7\%}} & {\color{gaingreen}$\uparrow$\,\textbf{28.9\%}} & {\color{lossred}$\downarrow$\,\textbf{2.1\%}} & {\color{gaingreen}$\uparrow$\,\textbf{5.1\%}} \\

\bottomrule
\end{tabular}
\end{table}

\begin{table}[!htb]
\centering
\setlength{\tabcolsep}{3pt}
\caption{\textbf{A second corpus.} LLaDA-8B-Instruct fine-tuned on LIFT-SFT-12K at an equal step budget. Conventions and markers follow Table~\ref{tab:main}.}
\label{tab:12k}
% AGDO means and sample SDs recomputed from three seeds.
\small
\begin{tabular}{l|cccc>{\columncolor{panelrow}}c}
\toprule
\textbf{Name} & \textbf{GSM8K} & \textbf{MATH} & \textbf{Countdown} & \textbf{Sudoku} & \textbf{Average} \\
\midrule
LLaDA-8B & 78.1 & 36.1 & 19.6 & 11.2 & 36.3 \\
Vanilla & \gmResultSD{\textbf{82.9}}{1.0} & \gmResultSD{34.6}{0.4} & \gmResultSD{19.9}{0.3} & \gmResultSD{9.3}{0.7} & \gmResultSD{36.7}{0.3} \\
GIFT & \gmResultSD{79.8}{0.3} & \gmResultSD{36.1}{0.1} & \gmResultSD{23.6}{4.7} & \gmResultSD{8.8}{1.2} & \gmResultSD{37.1}{1.0} \\
CART & \gmResultSD{80.2}{0.9} & \gmResultSD{36.4}{2.8} & \gmResultSD{22.4}{3.5} & \gmResultSD{15.6}{1.2} & \gmResultSD{38.6}{1.8} \\
LIFT$^{\ast}$ & \gmResultSD{81.3}{0.7} & \gmResultSD{36.6}{0.5} & \gmResultSD{23.6}{1.5} & \gmResultSD{11.3}{0.3} & \gmResultSD{38.2}{0.4} \\
AGDO$^\dagger$ & \gmResultSD{82.3}{0.5} & \gmResultSD{34.8}{0.4} & \gmResultSD{19.4}{1.7} & \gmResultSD{10.3}{0.9} & \gmResultSD{36.7}{0.3} \\
\rowcolor{goldrow}
\textbf{GoldiMask-R} & \gmResultSD{82.0}{0.6} & \gmResultSD{36.5}{0.7} & \gmResultSD{23.1}{1.9} & \gmResultSD{14.8}{0.1} & \gmResultSD{39.1}{0.9} \\
\rowcolor{goldrow}
{\scriptsize\color{gray}vs.\ CART} & {\color{gaingreen}$\uparrow$\,\textbf{2.2\%}} & {\color{gaingreen}$\uparrow$\,\textbf{0.3\%}} & {\color{gaingreen}$\uparrow$\,\textbf{3.1\%}} & {\color{lossred}$\downarrow$\,\textbf{5.1\%}} & {\color{gaingreen}$\uparrow$\,\textbf{1.3\%}} \\

\rowcolor{goldrow}
\textbf{GoldiMask-D} & \gmResultSD{82.2}{0.4} & \gmResultSD{\textbf{40.0}}{2.0} & \gmResultSD{\textbf{29.1}}{4.7} & \gmResultSD{\textbf{16.5}}{0.2} & \gmResultSD{\textbf{41.9}}{0.5} \\
\rowcolor{goldrow}
{\scriptsize\color{gray}vs.\ CART} & {\color{gaingreen}$\uparrow$\,\textbf{2.5\%}} & {\color{gaingreen}$\uparrow$\,\textbf{9.9\%}} & {\color{gaingreen}$\uparrow$\,\textbf{29.9\%}} & {\color{gaingreen}$\uparrow$\,\textbf{5.8\%}} & {\color{gaingreen}$\uparrow$\,\textbf{8.5\%}} \\

\bottomrule
\end{tabular}
\end{table}

The additional LLaDA-8B-Instruct experiment on LIFT-SFT-12K shows the same
pattern: both variants exceed the strongest baseline in average accuracy,
with GoldiMask-D achieving the highest average
(Table~\ref{tab:12k}). These results support the approach across the tested
backbones and training corpora, while individual benchmark outcomes vary.
Additional evaluation results are provided in Appendix~\ref{app:additional-results}.

\subsection{Extension to Code Generation}
\label{sec:code-results}

We next ask whether the benefits observed on reasoning tasks extend to code
generation. We report results for LLaDA-8B-Instruct and Dream-7B fine-tuned on
KodCode~\citep{xu2025kodcode} and evaluate pass@1 on
HumanEval~\citep{chen2021humaneval} and the sanitized MBPP test
split~\citep{austin2021mbpp}, using a maximum generation length of 256 tokens.

GoldiMask achieves the highest average accuracy across HumanEval and MBPP
among the compared fine-tuned methods, with GoldiMask-D leading on LLaDA-8B
and GoldiMask-R on Dream-7B (Table~\ref{tab:code}).

% Validation-selected results updated from author-approved rows on 2026-09-23.
% Retained rows and reused AGDO SDs are recorded in output/validation_migration_20260923.
\begin{table}[!htb]
\centering
\setlength{\tabcolsep}{10.4pt}
\caption{\textbf{Code generation.} HE denotes HumanEval; MBPP uses the sanitized test split.
Relative changes use each backbone's strongest fine-tuned baseline;
other conventions follow Table~\ref{tab:main}.}
\label{tab:code}
\small
\begin{tabular}{l|cc>{\columncolor{panelrow}}c|cc>{\columncolor{panelrow}}c}
\toprule
& \multicolumn{3}{c|}{\textbf{LLaDA-8B-Instruct}} & \multicolumn{3}{c}{\textbf{Dream-7B}} \\
\cmidrule(lr){2-4}\cmidrule(lr){5-7}
\textbf{Name} & \textbf{HE} & \textbf{MBPP} & \textbf{Avg.} & \textbf{HE} & \textbf{MBPP} & \textbf{Avg.} \\
\midrule
Base model & 34.1 & 41.3 & 37.7 & 39.6 & 52.5 & 46.1 \\
Vanilla SFT & \gmResultSD{30.8}{1.3} & \gmResultSD{41.6}{1.7} & \gmResultSD{36.2}{0.2} & \gmResultSD{45.4}{5.6} & \gmResultSD{42.4}{1.7} & \gmResultSD{43.9}{3.6} \\
GIFT & \gmResultSD{35.4}{1.7} & \gmResultSD{40.5}{2.8} & \gmResultSD{38.0}{0.5} & \gmResultSD{44.5}{1.7} & \gmResultSD{46.5}{0.8} & \gmResultSD{45.5}{0.4} \\
CART & \gmResultSD{33.5}{1.7} & \gmResultSD{39.9}{1.4} & \gmResultSD{36.7}{0.2} & \gmResultSD{47.5}{4.7} & \gmResultSD{50.4}{0.3} & \gmResultSD{49.0}{2.2} \\
AGDO$^\dagger$ & \gmResultSD{32.9}{0.9} & \gmResultSD{40.1}{2.2} & \gmResultSD{36.5}{0.7} & \gmResultSD{41.8}{3.0} & \gmResultSD{\textbf{53.7}}{5.5} & \gmResultSD{47.7}{4.3} \\
LIFT & \gmResultSD{33.5}{0.0} & \gmResultSD{39.5}{0.8} & \gmResultSD{36.5}{0.4} & \gmResultSD{42.4}{1.3} & \gmResultSD{49.4}{1.7} & \gmResultSD{45.9}{1.5} \\
\midrule
\rowcolor{goldrow}
\textbf{GoldiMask-R} & \gmResultSD{34.4}{0.4} & \gmResultSD{41.9}{4.4} & \gmResultSD{38.2}{2.4} & \gmResultSD{\textbf{49.1}}{3.0} & \gmResultSD{51.8}{0.6} & \gmResultSD{\textbf{50.4}}{1.2} \\
\rowcolor{goldrow}
{\scriptsize\color{gray}vs.\ best baseline} & {\color{lossred}\scriptsize $\downarrow$\,\textbf{2.8\%}} & {\color{gaingreen}\scriptsize $\uparrow$\,\textbf{3.5\%}} & {\color{gaingreen}\scriptsize $\uparrow$\,\textbf{0.5\%}} & {\color{gaingreen}\scriptsize $\uparrow$\,\textbf{3.4\%}} & {\color{gaingreen}\scriptsize $\uparrow$\,\textbf{2.8\%}} & {\color{gaingreen}\scriptsize $\uparrow$\,\textbf{2.9\%}} \\
\rowcolor{goldrow}
\textbf{GoldiMask-D} & \gmResultSD{\textbf{39.0}}{1.7} & \gmResultSD{\textbf{44.6}}{3.6} & \gmResultSD{\textbf{41.8}}{2.7} & \gmResultSD{45.4}{5.6} & \gmResultSD{49.6}{1.4} & \gmResultSD{47.5}{2.1} \\
\rowcolor{goldrow}
{\scriptsize\color{gray}vs.\ best baseline} & {\color{gaingreen}\scriptsize $\uparrow$\,\textbf{10.2\%}} & {\color{gaingreen}\scriptsize $\uparrow$\,\textbf{10.1\%}} & {\color{gaingreen}\scriptsize $\uparrow$\,\textbf{10.0\%}} & {\color{lossred}\scriptsize $\downarrow$\,\textbf{4.4\%}} & {\color{lossred}\scriptsize $\downarrow$\,\textbf{1.6\%}} & {\color{lossred}\scriptsize $\downarrow$\,\textbf{3.1\%}} \\
\bottomrule
\end{tabular}
\end{table}

\subsection{Accuracy and Parallel-Decoding Cost}
\label{sec:parallel-effects}
\label{sec:parallel-decoding}
% Retain the former subsubsection labels for existing cross-references.
\label{sec:parallel-theory}
\label{sec:parallel-results}
We compare GoldiMask-D/R/P with the adaptive SFT baselines using
dParallel's confidence-threshold decoder~\citep{chen2026dparallel} across
three backbones. Appendix~\ref{app:pd-setup} gives the setup.

\begin{figure}[!htb]
\centering
\includegraphics[width=\linewidth]{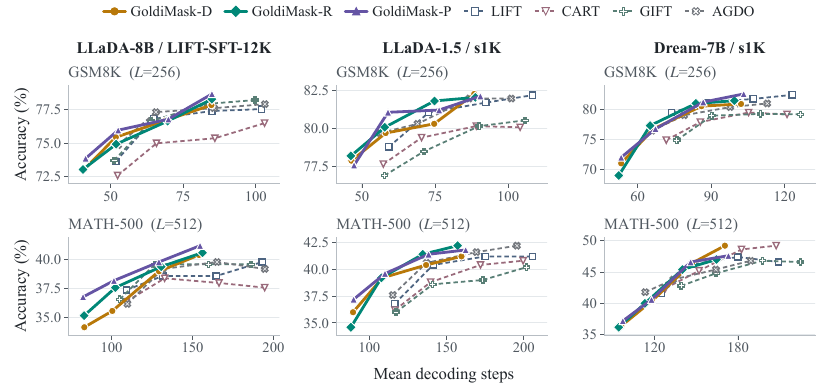}
\caption{\textbf{Accuracy versus parallel-decoding steps.}
GoldiMask-D/R/P (solid) and SFT baselines (dashed) on GSM8K and MATH-500.
Markers show $\tau=0.70,0.80,0.90,0.95$ from left to right; lines connect
measured settings. Upper-left is better. $L$ denotes generation length.}
\label{fig:pd-main}
\end{figure}

Figure~\ref{fig:pd-main} shows that GoldiMask consistently requires fewer
decoding steps on GSM8K and MATH-500 across the tested backbones and
thresholds, with comparable accuracy at higher thresholds. This pattern
holds for three reveal-set maximizers, showing that the decoding benefit
extends across GoldiMask variants. At lower thresholds, the reduction in
steps comes with a larger accuracy trade-off, particularly on Dream-7B.
Additional results on Countdown and Sudoku, with a discussion of
their constraint-based structure and its implications for parallel decoding,
appear in Appendix~\ref{app:pd-results}.

\begin{gmtakeaway}{1}
Improving parallel decoding also involves how the model is trained. Our
results suggest that adapting context and supervision during SFT can
reduce decoding iterations at comparable accuracy, even when the
decoding rule is unchanged.
\end{gmtakeaway}

\subsection{Ablation Study}
\label{sec:component-results}

We examine the contributions of selection and weighting and their design
choices on LLaDA-8B-Instruct fine-tuned on s1K.

\paragraph{Contributions of selection and weighting.}
We compare GoldiMask-D with random
reveals, static initial-marginal selection, and uniform target weights.
Static selection ranks candidates once by
$\Delta(j\mid\varnothing)=F(\{j\})-F(\varnothing)$ and reveals the top $B$
without updating their scores. Iterative greedy instead recomputes marginal
gains after each reveal, accounting for the changing context and remaining
targets. All reveal-based variants use the same reveal budget and two-pass
procedure; water-filling is retained except in the uniform-weight variant.
Vanilla SFT provides a reference.

\begin{table}[!htbp]
\centering\small
\setlength{\tabcolsep}{4.5pt}
\caption{\textbf{Contributions of selection and weighting.} Accuracy on
LLaDA-8B-Instruct fine-tuned on s1K. All reveal-based configurations use two passes.
Bold marks the best result in each column.}
\label{tab:results-components}
\begin{tabular}{@{}llcccc>{\columncolor{panelrow}}c@{}}
\toprule
\textbf{Reveal selection} & \textbf{Target weighting} & \textbf{GSM8K} & \textbf{MATH} & \textbf{Countdown} & \textbf{Sudoku} & \textbf{Avg.}\\
\midrule
None (vanilla SFT) & Uniform & 79.2 & 36.2 & 17.2 & 14.5 & 36.8 \\
Random & Water-filling & 80.1 & 37.5 & 24.6 & 14.5 & 39.2\\
Static initial-marginal & Water-filling & 80.2 & 37.8 & 24.6 & 14.9 & 39.4\\
GoldiMask-D & Uniform & 79.4 & 38.1 & 26.8 & 15.3 & 39.9\\
\rowcolor{goldrow}
GoldiMask-D & Water-filling & \textbf{80.4} & \textbf{39.5} & \textbf{27.7} & \textbf{16.9} & \textbf{41.1} \\
\bottomrule
\end{tabular}
\end{table}

GoldiMask-D achieves the highest average, 41.1
(Table~\ref{tab:results-components}), compared with 36.8 for Vanilla SFT.
With water-filling held fixed, random reveals reach 39.2 and static selection
reaches 39.4. Recomputing marginal gains therefore improves the average by
1.7 points over static ranking, supporting the benefit of adapting each
reveal to earlier choices. With iterative selection held fixed, uniform
weighting reaches 39.9, versus 41.1 with water-filling. Together, these
comparisons support both adaptive selection and target weighting in this setting.

\paragraph{Choice of supervision priority.}
We compare $\lambda(p)=p(1-p)^3$ with $1-p$, $p(1-p)$, and $p(1-p)^2$
(Table~\ref{tab:curves}). The three functions with an interior peak reach
averages of 40.5--41.1, compared with 38.6 for $1-p$, which gives the
highest priority to the least-confident targets. This comparison supports
prioritizing targets at intermediate confidence. Our choice has the highest
reported average, although the differences among the interior-peaked
functions are small. Appendix~\ref{app:gm-functions} provides the theoretical
and empirical motivation for the priority and contextual-support functions.

% Per-benchmark columns generated by scripts/make_ablation_tables.py in the
% GoldiMask code repo, from results_summary/all_eval_metrics.csv. Provenance:
%   lambda propto 1-p (pessimism, peak as p to 0): results/defaults_lam_1mp-best
%   lambda propto p(1-p) (symmetric, peak at 0.5): results/defaults_lam_p1p-best
%   lambda propto p(1-p)^2: corrected scores supplied by the author on 2026-09-22; four-task mean 40.65, displayed as 40.7.
%   measured shape, round constants (ours): author-approved validation-selected GoldiMask-D, LLaDA-8B+s1K (2026-09-23)
\begin{table}[!htbp]
\centering\small
\setlength{\tabcolsep}{4.5pt}
\caption{\textbf{Supervision-priority ablation.} Accuracy on LLaDA-8B-Instruct
fine-tuned on s1K, averaged over three seeds. Avg. is the mean across the four
benchmarks. Bold marks the highest average.}
\label{tab:curves}
\begin{tabular}{@{}lcccc>{\columncolor{panelrow}}c@{}}
\toprule
\textbf{Priority} & \textbf{GSM8K} & \textbf{MATH} & \textbf{Countdown} & \textbf{Sudoku} & \textbf{Avg.} \\
\midrule
$\lambda \propto 1-p$ (peak at $0$) & 79.2 & 38.6 & 20.7 & 15.9 & 38.6 \\
$\lambda \propto p(1-p)$ (peak at $1/2$) & 80.9 & 38.0 & 27.0 & 16.1 & 40.5 \\
$\lambda \propto p(1-p)^2$ (peak at $1/3$) & 81.0 & 38.4 & 26.3 & 16.9 & 40.7 \\
% fitted lookup tables for $\lambda$ and $\phi$             & -- & -- & -- & -- & 40.5 \\
\rowcolor{goldrow}
$\lambda \propto p(1-p)^3$ (ours, peak at $1/4$) & 80.4 & 39.5 & 27.7 & 16.9 & \textbf{41.1} \\
\bottomrule
\end{tabular}
\end{table}

\paragraph{Choice of loss weighting.}
Table~\ref{tab:weighting} compares uniform weighting, DRO-tilted hardness,
and utility-based weighting with caps. \emph{Uniform weighting} assigns
equal weight to every supervised token.

\emph{DRO-tilted hardness} gives more weight to difficult training examples.
Each example's loss is multiplied by $\exp(d/\tau)$, where $d$ is its mean
token difficulty and $\tau=0.25$. A running-average normalizer keeps the
multipliers approximately centered at one. This emphasis on difficult
examples is inspired by distributionally robust optimization (DRO).
% Author-supplied example-level DRO definition. The precise token-difficulty
% function and running-average update still need documentation from the run.

\emph{Uncapped utility weighting} normalizes the measured token utilities to
mean one without imposing a cap. For positive total utility, it assigns
$w_i=n u_i/\sum_{j\in\mathcal S}u_j$. This preserves the total weight $n$
while allowing a token with high utility to receive a large share of it.

\emph{Water-filling} distributes a fixed total weight in proportion to
measured token utility, subject to a per-token cap. It redistributes capped
mass among targets below the cap. We compare the default $c=10$ with $c=5$
and $c=15$.

Water-filling averages 41.1 with $c=10$, 40.9 with $c=5$, and 40.6 with
$c=15$. Uncapped utility weighting, DRO-tilted hardness, and uniform
weighting reach 39.4, 40.6, and 39.9, respectively.
Appendix~\ref{app:gm-weighting} explains the allocation properties of
water-filling and the role of the cap.

% Per-benchmark columns generated by scripts/make_ablation_tables.py in the
% GoldiMask code repo, from results_summary/all_eval_metrics.csv. Provenance:
%   uniform (no weighting): same run as tab:results-components' GoldiMask-D+uniform row; kept in sync
%   DRO-tilted hardness: results/v2_normattn_dro-best
%   water-filling (ours): author-approved validation-selected GoldiMask-D, LLaDA-8B+s1K (2026-09-23)
%   water-filling c=5: author-supplied means; Sudoku confirmed as 17.1 (2026-09-24), average 40.9.
%   water-filling c=15: author-corrected MATH 38.7 (2026-09-25), mean 40.625 displayed as 40.6.
%   uncapped utility weighting: author-corrected MATH 37.3 and Sudoku 14.7 (2026-09-24), mean 39.375 displayed as 39.4.
\begin{table}[!htbp]
\centering\small
\setlength{\tabcolsep}{4.5pt}
\caption{\textbf{Loss-weighting comparison.} Accuracy on LLaDA-8B-Instruct
fine-tuned on s1K, averaged over three seeds. Bold marks the highest average.}
\label{tab:weighting}
\begin{tabular}{@{}lcccc>{\columncolor{panelrow}}c@{}}
\toprule
\textbf{Weighting} & \textbf{GSM8K} & \textbf{MATH} & \textbf{Countdown} & \textbf{Sudoku} & \textbf{Avg.} \\
\midrule
Uniform & 79.4 & 38.1 & 26.8 & 15.3 & 39.9 \\
DRO-tilted hardness & 80.6 & 39.4 & 24.6 & 17.9 & 40.6 \\
Uncapped utility weighting & 78.4 & 37.3 & 27.1 & 14.7 & 39.4 \\
Water-filling ($c=5$) & 80.1 & 38.9 & 27.5 & 17.1 & 40.9 \\
Water-filling ($c=15$) & 80.1 & 38.7 & 27.9 & 15.8 & 40.6 \\
\rowcolor{goldrow}
Water-filling ($c=10$, ours) & 80.4 & 39.5 & 27.7 & 16.9 & \textbf{41.1} \\
\bottomrule
\end{tabular}
\end{table}
\FloatBarrier

\begin{gmtakeaway}{2}
Favoring targets at intermediate confidence matters more than precisely
locating the priority peak. Similarly, limiting
token weights matters more than the exact cap, with similar average
accuracy for $c=5,10,15$. The priority curve and cap can therefore be chosen
flexibly within the tested range.
\end{gmtakeaway}

\section{Conclusion}
\label{sec:conclusion}
In this work, we introduced GoldiMask, an approach to supervised fine-tuning
of discrete diffusion language models that jointly considers which tokens
provide context and how strongly the remaining targets are supervised. By
combining submodular context selection with response-based loss weighting,
GoldiMask improves average benchmark accuracy across the evaluated backbones
and training corpora. Its benefits also extend to parallel decoding, often
yielding comparable accuracy with fewer decoding iterations.
% We also developed a single-pass variant that retains part of these gains with little
% additional training cost, while the full approach achieves higher average accuracy.
Together, these results highlight context selection and supervision
weighting as complementary tools for improving diffusion SFT.

\section{Limitations and Future Work}
\label{sec:limitations}

GoldiMask demonstrates the benefits of coordinating context selection and
supervision weighting, and provides a starting point for improving both.
Our theoretical analysis and empirical studies support the chosen functions,
though they do not establish their optimality. The full method trades
additional training cost for higher average accuracy.
% The single-pass GoldiMask-A variant retains part of the gains with little overhead.

A promising next step is to develop more direct estimates of how much each
token benefits from supervision and how this depends on its context. Such
estimates could guide both target selection and loss weighting. The
parallel-decoding gains also motivate studying training and decoding
together: beyond the confidence-threshold procedure evaluated here, future
work could explore how supervision strategies support different ways of
generating multiple tokens at once.

\FloatBarrier
% Conference-specific statements are omitted from the preprint edition.
\clearpage
\bibliographystyle{arxiv}
\bibliography{arxiv}

\clearpage
\appendix
\section{Extended related work}
\label{app:related}

See Section~\ref{sec:related} for the directly comparable methods. \paragraph{Trajectory- and order-based approaches.} TABOM \citep{chen2026tabom} post-trains on the model's own decoding trajectories with an entropy-ranking loss. It starts, independently, from the
same diagnosis as our reveal design: random masking during training does not match the easy-to-hard order of inference. The remedies differ. TABOM changes what the model trains on, while we change how
the response is corrupted. Where-to-Unmask \citep{asano2026unmask} learns an inference-time unmasking planner over a frozen token model, complementary to training-time corruption design. \paragraph{Curriculum and learnability.} Curriculum learning \citep{bengio2009curriculum}, self-paced learning \citep{kumar2010self}, and learning-progress curricula \citep{graves2017automated} formalize the intuition that training is most effective at intermediate difficulty \citep{kidd2012goldilocks,wilson2019eighty}. Our function examines how a target's response to additional context varies with its initial confidence, providing empirical motivation for confidence-dependent context allocation. RHO-LOSS \citep{mindermann2022prioritized} and Rho-1 \citep{lin2024rho} select training points that are learnable and not yet learnt, at example and token granularity via reference models; our estimand is measured by within-token contrasts instead.

\paragraph{Submodular selection.} Submodular functions model diminishing returns, where the marginal gain of adding an element decreases as the selected set grows. In our setting, revealing a token provides context to the remaining prediction targets, with diminishing benefit as their contextual support increases. Revealing that token also removes it from supervision. We capture this interaction through a submodular objective that sums concave functions of accumulated attention support over the remaining targets, weighted by their supervision priority. Because each reveal also removes a target's contribution, the objective is generally non-monotone. More broadly, submodular objectives are widely used in machine learning to select representative and non-redundant subsets, including data summarization~\cite{chen2024discretely,mualem2024submodular}, active learning~\cite{chen2013near}, feature selection~\cite{tukan2024practical}, model compression~\cite{el2022data}, and sensor placement~\citep{zhou2022risk,tukan2023orbslam3}. For monotone submodular maximization under a cardinality constraint, greedy optimization admits the classical $(1-1/e)$ approximation guarantee~\citep{nemhauser1978analysis}.

\paragraph{Parallel decoding.} 
Autoregressive language models restrict each prediction to the left context during both training and inference, whereas the bidirectional attention of discrete diffusion language models offers a wider range of possibilities. Precisely because many tokens can be decoded at the same time, there is the promise of making inference faster. Methods that enhance faster parallel decoding include certainty-forcing distillation as in dParallel \citep{chen2026dparallel}, and training-free rules based on confidence (Fast-dLLM and DAWN), entropy (EB-Sampler), temporal stability (KLASS), or locality (LocalLeap) \citep{wu2026fastdllm, kim2026klass, ben2026accelerated-ebsampler, kong2025accelerating-localleap}. DAWN \citep{luo2026dawn} combines confidence with attention-based dependencies, revealing tokens supported by previously revealed high-confidence anchors while avoiding coupled uncertain tokens. AXON \citep{ayoub2026axon} instead uses confidence to identify reliable anchors and attention to select those most useful to the remaining uncertain positions. Diffusion and autoregressive models can also be combined through speculative decoding \citep{leviathan2023fast}: DFlash \citep{chen2026dflash} uses diffusion models to obtain many tokens in parallel before an autoregressive model verifies them.

\section{Main Algorithm}
\label{app:main-algorithm}

Algorithm~\ref{alg:goldimask} describes the exact-budget GoldiMask-D/R variants.
We sample $\rho\sim\mathcal{U}(0.1,1-t)$ for $t<0.9$ and
$\rho\sim\mathcal{U}(0,1-t)$ otherwise, applying the over-mask-then-reveal
procedure in both regimes. If $|\mathcal C|\leq K$, we reveal no tokens and
apply uniformly weighted supervision to all masked candidates, retaining
the loss normalization $1/(tL)$.

% \begin{algorithm}[t]
% \caption{GoldiMask, one supervised fine-tuning step. The two measured curves
% $\lambda$ and $\phi$ enter as inputs and are fixed before training.}
% \label{alg:goldimask}
% \begin{algorithmic}[1]
% \Require response $x$ of length $L$; model $f_\theta$; learnability curve
% $\lambda(\cdot)$; support curve $\phi(\cdot)$; weight cap $c$
% Historical pseudocode below predates the piecewise sampling rule.
% \State $M \gets$ random subset of $\{1,\dots,L\}$ with $|M| = \lfloor tL \rfloor$
% \State $B \gets \lfloor \rho L \rfloor$
% \State $p^{1}, A \gets f_\theta(x \setminus M)$
% \State $R \gets \textsc{Select}(F,\, M,\, B)$
%   \Comment{any maximizer of submodular $F$ s.t.\ $|R| = B$; App.~\ref{sec:selectors}}
% \State $S \gets M \setminus R$
% \State $p^{2}, \ell \gets f_\theta(x \setminus S)$
% \For{$i \in S$}
%   \State $u_i \gets \max\!\big(\log p^{2}_i - \log p^{1}_i,\; 0\big)\cdot \lambda(p^{2}_i)$
% \EndFor
% \State find $\nu$ by bisection so that $\sum_{i \in S} \min(u_i/\nu,\, c) = |S|$
% \State $w_i \gets \min(u_i/\nu,\, c)$ \quad for $i \in S$
% \State \Return $\dfrac{1}{t}\sum_{i \in S} w_i \, \ell_i$
% \Statex \hskip-\algorithmicindent
%   \textbf{where}\quad $F(R) \;=\; \displaystyle\sum_{i \,\in\, M \setminus R}
%   \lambda\big(p^{1}_i\big)\;\phi\big(S_i(R)\big)$,
%   \qquad $S_i(R) = \displaystyle\sum_{j \in R} A_{ij}$
% \end{algorithmic}
% \end{algorithm}
\begin{algorithm}[htb!]
\caption{GoldiMask, one supervised fine-tuning step. }
\label{alg:goldimask}
\begin{algorithmic}[1]
\Require response $x$ of length $L$; model $f_\theta$;
learnability curve $\lambda(\cdot)$; support curve $\phi(\cdot)$; weight cap $c>1$
\State $t \sim \mathcal{U}(0,1)$
\If{$t<0.9$}
  \State $\rho \sim \mathcal{U}(0.1,\,1-t)$
\Else
  \State $\rho \sim \mathcal{U}(0,\,1-t)$
\EndIf
\State $\mathcal{C} \gets$ positions masked by the forward process at rate $t+\rho$
\State $K \gets \lfloor tL \rfloor$, \quad $B \gets |\mathcal{C}| - K$
\If{$|\mathcal C|\leq K$}
  \State $p \gets f_\theta(x\setminus\mathcal C)$ \Comment{no reveal; uniform weights}
  \State \Return $-\dfrac{1}{tL}\sum_{i\in\mathcal C}\log p_i$
\EndIf
\State $p^{(1)}, A \gets f_\theta(x \setminus \mathcal{C})$
\State $R \gets \textsc{Select}(F,\, \mathcal{C},\, B)$
  \Comment{any maximizer of submodular $F$ s.t. $|R| = B$; App.~\ref{sec:selectors}}
\State $\mathcal{S} \gets \mathcal{C} \setminus R$
\State $p^{(2)} \gets f_\theta(x \setminus \mathcal{S})$
\For{$i \in \mathcal{S}$}
  \State $u_i \gets \big[\log p^{(2)}_i - \log p^{(1)}_i\big]_{+}\cdot \lambda\big(p^{(2)}_i\big)$
\EndFor
\State $P \gets \{i\in\mathcal S : u_i>0\}$, \quad $n_+\gets|P|$ \Comment{water-filling}
\If{$n_+=0$}
  \State $w_i\gets1$ for all $i\in\mathcal S$
\ElsIf{$c n_+\leq K$}
  \State $w_i\gets c$ for $i\in P$
  \State $w_i\gets (K-c n_+)/(K-n_+)$ for $i\in\mathcal S\setminus P$
\Else
  \State find $\nu>0$ by bisection so that $\sum_{i\in P}\min(u_i/\nu,\,c)=K$
  \State $w_i\gets\min(u_i/\nu,\,c)$ for all $i\in\mathcal S$
\EndIf
\State Detach the weights $w_i$ from the computation graph.
\State \Return $-\dfrac{1}{tL}\sum_{i \in \mathcal{S}} w_i \log p^{(2)}_i$
\Statex \hskip-\algorithmicindent
  \textbf{where}\quad $F(R) = \displaystyle\sum_{i \in \mathcal{C} \setminus R}
  \lambda\big(p^{(1)}_i\big)\,\phi\big(S_i(R)\big)$,
  \qquad $S_i(R) = \displaystyle\sum_{j \in R} A_{ij}$
\end{algorithmic}
\end{algorithm}
\paragraph{Maximizing the reveal objective.}

The selection step in Algorithm~\ref{alg:goldimask} constructs $R$ greedily on the marginals
of $F$. Because $F$ is submodular but
non-monotone, the guarantee usually quoted for greedy selection does not apply
here. The selector implementations and their budget conventions are
given in Appendix~\ref{app:arms}.

\section{Reveal-Set Selection}
\label{app:reveal-selection}
This section explains how to optimize the reveal objective and explores how the choice of the selection algorithm affects the corresponding downstream performance. Section~\ref{app:arms} introduces the four evaluated algorithms and their budget conventions. Section \ref{app:gm-selection-properties} then develops the mathematical basis for selection. It proves that the objective is submodular, explains why it is generally non-monotone, and derives the exact marginal gain used to score each candidate reveal.
% This section brings together the choice of maximizer, its empirical
% comparison, and the structural properties of the reveal objective. A reveal
% supplies context to other targets while removing its own supervised
% contribution. The resulting objective has diminishing returns but need not
% increase with every reveal. Proposition~\ref{prop:submodular} formalizes this
% property, and Equation~\ref{eq:gm-reveal-marginal} gives the marginal used by
% the selector. We first describe and compare the maximizers, then give the
% formal derivation; guarantees for a particular algorithm require its own
% constraint assumptions.

% appendix/arms.tex — the four submodular maximizers used for the reveal objective.
% Created 2026-09-12 (Claude session). Included by appendix/reveal_selection.tex.
%
% New bib entries needed (buchbinder2014submodular and nemhauser1978analysis already exist):
%
% @inproceedings{han2020twingreedy,
%   title={Deterministic Approximation for Submodular Maximization over a Matroid in Nearly Linear Time},
%   author={Han, Kai and Cao, Zongmai and Cui, Shuang and Wu, Benwei},
%   booktitle={Advances in Neural Information Processing Systems},
%   volume={33},
%   year={2020}
% }
% @inproceedings{tukan2024practical,
%   title={Practical $0.385$-Approximation for Submodular Maximization Subject to a Cardinality Constraint},
%   author={Tukan, Murad and Mualem, Loay and Feldman, Moran},
%   booktitle={Advances in Neural Information Processing Systems},
%   volume={37},
%   pages={51223--51253},
%   year={2024}
% }

\subsection{Maximizers for the reveal objective}
\label{sec:selectors}
\label{app:arms}

The reveal objective $F(R)=\sum_{i\in\mathcal C\setminus R}\lambda(p_i)\,\phi(S_i(R))$ is
submodular but not monotone (Proposition~\ref{prop:submodular}): revealing a token adds support to
the targets that attend to it and, at the same time, removes that token's own term from the sum, so
$F(\varnothing)=F(\mathcal C)=0$ and the value can fall as the reveal set grows.
We run GoldiMask with four maximizers with different determinism and budget characteristics. Table~\ref{tab:arms} summarizes the four methods.
Table~\ref{tab:arms-full} compares their performance across all reasoning settings against the strongest baseline in each setting.

\definecolor{armhead}{RGB}{52,52,52}
\definecolor{armband}{RGB}{247,247,247}

\begin{table}[!b]
\centering
\small
\setlength{\tabcolsep}{4pt}
\renewcommand{\arraystretch}{1.25}
\caption{\textbf{The four maximizers used for reveal-set selection.}
The selection procedures below describe
our implementations.}
\label{tab:arms}
\resizebox{\ifdim\width>\linewidth\linewidth\else\width\fi}{!}{%
\begin{tabular}{@{}l c l@{}}
\toprule
\rowcolor{armhead}
\textcolor{white}{\textbf{Maximizer}} &
\textcolor{white}{\textbf{Deterministic}} &
\textcolor{white}{\textbf{Reference}} \\
\midrule
\textbf{Greedy} (D) & yes &
\citet{nemhauser1978analysis}, Math.\ Programming \\
\rowcolor{armband}
\textbf{Randomized greedy} (R) & no &
\citet{buchbinder2014submodular}, SODA \\
\rowcolor{goldrow}
\textbf{TwinGreedy} (T) & yes &
\citet{han2020deterministic}, NeurIPS \\
\rowcolor{armband}
\textbf{Practical $0.385$} (P) & no &
\citet{tukan2024practical}, NeurIPS \\
\bottomrule
\end{tabular}}
\end{table}

\begin{samepage}
\paragraph{Selection procedures.}
The four maximizers construct reveal sets as follows.
\begin{itemize}
\item \textbf{Greedy} adds the candidate with the largest marginal gain at
each step (Equation~\ref{eq:gm-reveal-marginal}), recomputing gains as the
reveal set grows.
\item \textbf{Randomized greedy} samples uniformly from the $B$ highest-marginal
remaining candidates, updating the pool after each selection.
\item \textbf{TwinGreedy} grows two disjoint candidate sets. At each step,
it chooses the feasible token--set pair with the largest marginal gain
and adds that token to that set. It returns the set with the higher
objective value.
\item \textbf{Practical $0.385$} (P) retains the original algorithm's
three-stage structure. We construct a local-search guide $Z$, run guided
stochastic greedy while excluding $Z$ during the first $\lceil0.372B\rceil$
iterations, and return whichever solution has the higher objective value.
Our implementation uses exact best-swap sweeps to construct the guide and
samples uniformly from the $B$ highest-marginal candidates in each greedy
iteration, skipping negative-gain additions.
\end{itemize}
All four use the same marginal-gain oracle, priority and support functions.
GoldiMask-D/R use an exact reveal budget, while GoldiMask-T/P may terminate
early and retain additional supervised targets. All variants use mean-one
water-filling over the remaining targets.
\end{samepage}

\paragraph{All arms on all settings.}
Table~\ref{tab:arms-full} places the four maximizers side by side on the four settings of
Tables~\ref{tab:main} and~\ref{tab:12k}, against the strongest baseline of each setting.

\begin{table}[htb!]
\setlength{\tabcolsep}{4.5pt}
\caption{\textbf{All maximizers on all settings.} Same benchmarks, protocol and conventions as
Tables~\ref{tab:main} and~\ref{tab:12k}. The baseline row is the strongest baseline of each setting
by four-benchmark average; GoldiMask-D, GoldiMask-R and baseline rows are those of the main text.
GoldiMask-T and -P are two-seed means using validation-selected checkpoints. The relative change beneath each
GoldiMask row is measured against that baseline. Bold marks the best entry per column within
a setting.}
\label{tab:arms-full}
\begin{center}
\small
\begin{tabular}{l|cccc>{\columncolor{panelrow}}c}
\toprule
\textbf{Name} & \textbf{GSM8K} & \textbf{MATH} & \textbf{Countdown} & \textbf{Sudoku} & \textbf{Average} \\
\midrule
\multicolumn{6}{c}{\textbf{LLaDA-8B-Instruct / s1K}}\\
\midrule
LIFT$^{\ast}$ & 79.1 & 36.6 & 28.8 & 13.5 & 39.5 \\
\rowcolor{goldrow}
\textbf{GoldiMask-D} & 80.4 & \textbf{39.5} & 27.7 & 16.9 & \textbf{41.1} \\
\rowcolor{goldrow}
{\scriptsize\color{gray}vs.\ LIFT} & {\color{gaingreen}$\uparrow$\,\textbf{1.6\%}} & {\color{gaingreen}$\uparrow$\,\textbf{7.9\%}} & {\color{lossred}$\downarrow$\,\textbf{3.8\%}} & {\color{gaingreen}$\uparrow$\,\textbf{25.2\%}} & {\color{gaingreen}$\uparrow$\,\textbf{4.1\%}} \\

\rowcolor{goldrow}
\textbf{GoldiMask-R} & 80.4 & 38.3 & \textbf{29.7} & 14.4 & 40.7 \\
\rowcolor{goldrow}
{\scriptsize\color{gray}vs.\ LIFT} & {\color{gaingreen}$\uparrow$\,\textbf{1.6\%}} & {\color{gaingreen}$\uparrow$\,\textbf{4.6\%}} & {\color{gaingreen}$\uparrow$\,\textbf{3.1\%}} & {\color{gaingreen}$\uparrow$\,\textbf{6.7\%}} & {\color{gaingreen}$\uparrow$\,\textbf{3.0\%}} \\

\rowcolor{goldrow}
\textbf{GoldiMask-T} & \textbf{81.5} & 38.7 & 26.8 & 15.1 & 40.5 \\
\rowcolor{goldrow}
{\scriptsize\color{gray}vs.\ LIFT} & {\color{gaingreen}$\uparrow$\,\textbf{3.0\%}} & {\color{gaingreen}$\uparrow$\,\textbf{5.7\%}} & {\color{lossred}$\downarrow$\,\textbf{6.9\%}} & {\color{gaingreen}$\uparrow$\,\textbf{11.9\%}} & {\color{gaingreen}$\uparrow$\,\textbf{2.5\%}} \\

\rowcolor{goldrow}
\textbf{GoldiMask-P} & 81.0 & 36.5 & 28.8 & \textbf{17.4} & 40.9 \\
\rowcolor{goldrow}
{\scriptsize\color{gray}vs.\ LIFT} & {\color{gaingreen}$\uparrow$\,\textbf{2.4\%}} & {\color{lossred}$\downarrow$\,\textbf{0.3\%}} & {\color{gray}\textbf{0.0\%}} & {\color{gaingreen}$\uparrow$\,\textbf{28.9\%}} & {\color{gaingreen}$\uparrow$\,\textbf{3.5\%}} \\

\midrule
\multicolumn{6}{c}{\textbf{LLaDA-1.5 / s1K}}\\
\midrule
LIFT$^{\ast}$ & 80.4 & 37.7 & 25.7 & 16.1 & 40.0 \\
\rowcolor{goldrow}
\textbf{GoldiMask-D} & 81.9 & 37.6 & 26.6 & \textbf{21.6} & 41.9 \\
\rowcolor{goldrow}
{\scriptsize\color{gray}vs.\ LIFT} & {\color{gaingreen}$\uparrow$\,\textbf{1.9\%}} & {\color{lossred}$\downarrow$\,\textbf{0.3\%}} & {\color{gaingreen}$\uparrow$\,\textbf{3.5\%}} & {\color{gaingreen}$\uparrow$\,\textbf{34.2\%}} & {\color{gaingreen}$\uparrow$\,\textbf{4.7\%}} \\

\rowcolor{goldrow}
\textbf{GoldiMask-R} & 81.7 & 40.5 & 29.9 & 17.9 & \textbf{42.5} \\
\rowcolor{goldrow}
{\scriptsize\color{gray}vs.\ LIFT} & {\color{gaingreen}$\uparrow$\,\textbf{1.6\%}} & {\color{gaingreen}$\uparrow$\,\textbf{7.4\%}} & {\color{gaingreen}$\uparrow$\,\textbf{16.3\%}} & {\color{gaingreen}$\uparrow$\,\textbf{11.2\%}} & {\color{gaingreen}$\uparrow$\,\textbf{6.2\%}} \\

\rowcolor{goldrow}
\textbf{GoldiMask-T} & \textbf{82.3} & 40.5 & 26.6 & 17.6 & 41.7 \\
\rowcolor{goldrow}
{\scriptsize\color{gray}vs.\ LIFT} & {\color{gaingreen}$\uparrow$\,\textbf{2.4\%}} & {\color{gaingreen}$\uparrow$\,\textbf{7.4\%}} & {\color{gaingreen}$\uparrow$\,\textbf{3.5\%}} & {\color{gaingreen}$\uparrow$\,\textbf{9.3\%}} & {\color{gaingreen}$\uparrow$\,\textbf{4.2\%}} \\

\rowcolor{goldrow}
\textbf{GoldiMask-P} & 81.7 & \textbf{40.9} & \textbf{30.3} & 16.7 & 42.4 \\
\rowcolor{goldrow}
{\scriptsize\color{gray}vs.\ LIFT} & {\color{gaingreen}$\uparrow$\,\textbf{1.6\%}} & {\color{gaingreen}$\uparrow$\,\textbf{8.5\%}} & {\color{gaingreen}$\uparrow$\,\textbf{17.9\%}} & {\color{gaingreen}$\uparrow$\,\textbf{3.7\%}} & {\color{gaingreen}$\uparrow$\,\textbf{6.0\%}} \\

\midrule
\multicolumn{6}{c}{\textbf{Dream-7B / s1K}}\\
\midrule
LIFT$^{\ast}$ & 78.4 & 42.1 & 28.7 & 23.3 & 43.1 \\
\rowcolor{goldrow}
\textbf{GoldiMask-D} & 78.8 & 42.8 & 37.0 & 22.8 & 45.3 \\
\rowcolor{goldrow}
{\scriptsize\color{gray}vs.\ LIFT} & {\color{gaingreen}$\uparrow$\,\textbf{0.5\%}} & {\color{gaingreen}$\uparrow$\,\textbf{1.7\%}} & {\color{gaingreen}$\uparrow$\,\textbf{28.9\%}} & {\color{lossred}$\downarrow$\,\textbf{2.1\%}} & {\color{gaingreen}$\uparrow$\,\textbf{5.1\%}} \\

\rowcolor{goldrow}
\textbf{GoldiMask-R} & \textbf{79.6} & \textbf{44.3} & 33.1 & \textbf{24.3} & 45.3 \\
\rowcolor{goldrow}
{\scriptsize\color{gray}vs.\ LIFT} & {\color{gaingreen}$\uparrow$\,\textbf{1.5\%}} & {\color{gaingreen}$\uparrow$\,\textbf{5.2\%}} & {\color{gaingreen}$\uparrow$\,\textbf{15.3\%}} & {\color{gaingreen}$\uparrow$\,\textbf{4.3\%}} & {\color{gaingreen}$\uparrow$\,\textbf{5.1\%}} \\

\rowcolor{goldrow}
\textbf{GoldiMask-T} & 78.7 & 43.7 & \textbf{41.8} & 24.2 & \textbf{47.1} \\
\rowcolor{goldrow}
{\scriptsize\color{gray}vs.\ LIFT} & {\color{gaingreen}$\uparrow$\,\textbf{0.4\%}} & {\color{gaingreen}$\uparrow$\,\textbf{3.8\%}} & {\color{gaingreen}$\uparrow$\,\textbf{45.6\%}} & {\color{gaingreen}$\uparrow$\,\textbf{3.9\%}} & {\color{gaingreen}$\uparrow$\,\textbf{9.3\%}} \\

\rowcolor{goldrow}
\textbf{GoldiMask-P} & 78.7 & 44.0 & 40.8 & 21.1 & 46.2 \\
\rowcolor{goldrow}
{\scriptsize\color{gray}vs.\ LIFT} & {\color{gaingreen}$\uparrow$\,\textbf{0.4\%}} & {\color{gaingreen}$\uparrow$\,\textbf{4.5\%}} & {\color{gaingreen}$\uparrow$\,\textbf{42.2\%}} & {\color{lossred}$\downarrow$\,\textbf{9.4\%}} & {\color{gaingreen}$\uparrow$\,\textbf{7.2\%}} \\

\midrule
\multicolumn{6}{c}{\textbf{LLaDA-8B-Instruct / LIFT-SFT-12K}}\\
\midrule
CART & 80.2 & 36.4 & 22.4 & 15.6 & 38.6 \\
\rowcolor{goldrow}
\textbf{GoldiMask-D} & \textbf{82.2} & \textbf{40.0} & 29.1 & \textbf{16.5} & \textbf{41.9} \\
\rowcolor{goldrow}
{\scriptsize\color{gray}vs.\ CART} & {\color{gaingreen}$\uparrow$\,\textbf{2.5\%}} & {\color{gaingreen}$\uparrow$\,\textbf{9.9\%}} & {\color{gaingreen}$\uparrow$\,\textbf{29.9\%}} & {\color{gaingreen}$\uparrow$\,\textbf{5.8\%}} & {\color{gaingreen}$\uparrow$\,\textbf{8.5\%}} \\

\rowcolor{goldrow}
\textbf{GoldiMask-R} & 82.0 & 36.5 & 23.1 & 14.8 & 39.1 \\
\rowcolor{goldrow}
{\scriptsize\color{gray}vs.\ CART} & {\color{gaingreen}$\uparrow$\,\textbf{2.2\%}} & {\color{gaingreen}$\uparrow$\,\textbf{0.3\%}} & {\color{gaingreen}$\uparrow$\,\textbf{3.1\%}} & {\color{lossred}$\downarrow$\,\textbf{5.1\%}} & {\color{gaingreen}$\uparrow$\,\textbf{1.3\%}} \\

\rowcolor{goldrow}
\textbf{GoldiMask-T} & 81.7 & 39.0 & \textbf{30.1} & 14.0 & 41.2 \\
\rowcolor{goldrow}
{\scriptsize\color{gray}vs.\ CART} & {\color{gaingreen}$\uparrow$\,\textbf{1.9\%}} & {\color{gaingreen}$\uparrow$\,\textbf{7.1\%}} & {\color{gaingreen}$\uparrow$\,\textbf{34.4\%}} & {\color{lossred}$\downarrow$\,\textbf{10.3\%}} & {\color{gaingreen}$\uparrow$\,\textbf{6.7\%}} \\

\rowcolor{goldrow}
\textbf{GoldiMask-P} & 81.0 & 37.1 & 29.1 & 14.1 & 40.3 \\
\rowcolor{goldrow}
{\scriptsize\color{gray}vs.\ CART} & {\color{gaingreen}$\uparrow$\,\textbf{1.0\%}} & {\color{gaingreen}$\uparrow$\,\textbf{1.9\%}} & {\color{gaingreen}$\uparrow$\,\textbf{29.9\%}} & {\color{lossred}$\downarrow$\,\textbf{9.6\%}} & {\color{gaingreen}$\uparrow$\,\textbf{4.4\%}} \\

\bottomrule
\end{tabular}
\end{center}
\end{table}

\subsection{Properties of the reveal objective}
\label{app:gm-selection-properties}

Let $C$ be a finite set of candidate tokens. For every pair $i,j\in C$,
let $A_{ij}\geq 0$ denote the attention support that revealing token $j$
delivers to token $i$. We assume that $A$ has zero diagonal and that the
attention matrix is fixed throughout the construction of the reveal set.
Likewise, let $\lambda_i\geq 0$ be the fixed learnability weight of token $i$.

For a reveal set $R\subseteq C$, define the total support delivered to token
$i$ by
\[
S_i(R)
\;=\;
\sum_{j\in R}A_{ij}.
\]
The GoldiMask objective is
\[
F(R)
\;=\;
\sum_{i\in C\setminus R}
\lambda_i\,\phi\!\bigl(S_i(R)\bigr),
\qquad R\subseteq C.
\]
The complement $C\setminus R$ appears because the tokens in $R$ are revealed
and therefore are not supervised. Hence, revealing a token can help the
remaining supervised tokens while simultaneously removing its own contribution
from the objective.

A function $F:2^C\to\mathbb{R}$ is \emph{submodular} if it satisfies the
diminishing-returns property. For $R\subseteq C$ and $x\in C\setminus R$,
define the marginal change from adding $x$ to $R$ as
\[
\Delta(x\mid R)
\;:=\;
F(R\cup\{x\})-F(R).
\]
Then $F$ is submodular if
\[
\Delta(x\mid R)
\;\geq\;
\Delta(x\mid R')
\qquad
\text{for every }
R\subseteq R'\subseteq C
\text{ and }
x\notin R'.
\]
This means that once more elements have already been selected, adding $x$ cannot
become more valuable.

\begin{proposition}[Submodularity]
\label{prop:submodular}
Suppose that $\phi:[0,\infty)\to\mathbb{R}$ is concave and nondecreasing.
Then
\[
F(R)
=
\sum_{i\in C\setminus R}
\lambda_i\,\phi\!\left(\sum_{j\in R}A_{ij}\right)
\]
is submodular on $2^C$. If, in addition, $\phi(0)=0$, then
$F(\emptyset)=0$ and $F(R)\geq 0$ for every $R\subseteq C$.
\end{proposition}

\begin{proof}
We prove the diminishing-returns inequality
\[
\Delta(x\mid R)
\;\geq\;
\Delta(x\mid R')
\]
for arbitrary sets $R\subseteq R'\subseteq C$ and an arbitrary
$x\notin R'$.

\medskip
\noindent
\emph{Step 1: Diminishing increments of the concave response function.}

For any fixed $a\geq 0$, define
\[
d_a(s)
\;:=\;
\phi(s+a)-\phi(s),
\qquad s\geq 0.
\]
The quantity $d_a(s)$ is the benefit of adding $a$ units of support when the
current support level is $s$.

Because $\phi$ is concave, its increments decrease as the starting point
increases. Therefore,
\[
s\leq t
\qquad\Longrightarrow\qquad
d_a(s)\geq d_a(t).
\]
Equivalently,
\[
\phi(s+a)-\phi(s)
\;\geq\;
\phi(t+a)-\phi(t)
\qquad
\text{whenever }s\leq t.
\]
Because $\phi$ is also nondecreasing and $a\geq 0$, we have
\[
d_a(s)
=
\phi(s+a)-\phi(s)
\geq 0.
\]
Thus, for every fixed support increment $a\geq 0$, the function $d_a$ is
nonnegative and nonincreasing.

\medskip
\noindent
\emph{Step 2: Derivation of the exact marginal change.}

Fix a set $R\subseteq C$ and an element $x\notin R$. For notational
convenience, write
\[
R_x
\;:=\;
R\cup\{x\}.
\]
For every token $i$, adding $x$ to the reveal set changes its support according
to
\[
\begin{aligned}
S_i(R_x)
&=
\sum_{j\in R_x}A_{ij}
&=
\sum_{j\in R}A_{ij}+A_{ix}
&=
S_i(R)+A_{ix}.
\end{aligned}
\]

By definition,
\[
F(R_x)
=
\sum_{i\in C\setminus R_x}
\lambda_i\,\phi\!\bigl(S_i(R_x)\bigr).
\]
Since $x\notin R$, the objective before revealing $x$ can be decomposed as
\[
F(R)
=
\lambda_x\phi\!\bigl(S_x(R)\bigr)
+
\sum_{i\in C\setminus R_x}
\lambda_i\,\phi\!\bigl(S_i(R)\bigr).
\]
The first term is the contribution of $x$ while it is still supervised. The
remaining sum contains the tokens that stay supervised both before and after
$x$ is revealed.

Therefore,
\[
\begin{aligned}
\Delta(x\mid R)
&=
F(R_x)-F(R) \\
&=
\sum_{i\in C\setminus R_x}
\lambda_i\,\phi\!\bigl(S_i(R_x)\bigr)
-
\left[
\lambda_x\phi\!\bigl(S_x(R)\bigr)
+
\sum_{i\in C\setminus R_x}
\lambda_i\,\phi\!\bigl(S_i(R)\bigr)
\right] \\
&=
\sum_{i\in C\setminus R_x}
\lambda_i
\left[
\phi\!\bigl(S_i(R_x)\bigr)
-
\phi\!\bigl(S_i(R)\bigr)
\right]
-
\lambda_x\phi\!\bigl(S_x(R)\bigr).
\end{aligned}
\]
Using
\[
S_i(R_x)=S_i(R)+A_{ix},
\]
we obtain
\[
\begin{aligned}
\Delta(x\mid R)
&=
\sum_{i\in C\setminus(R\cup\{x\})}
\lambda_i
\left[
\phi\!\bigl(S_i(R)+A_{ix}\bigr)
-
\phi\!\bigl(S_i(R)\bigr)
\right]
-
\lambda_x\phi\!\bigl(S_x(R)\bigr) \\
&=
\sum_{i\in C\setminus(R\cup\{x\})}
\lambda_i\,
d_{A_{ix}}\!\bigl(S_i(R)\bigr)
-
\lambda_x\phi\!\bigl(S_x(R)\bigr).
\end{aligned}
\]
Hence the marginal has two components:
\begin{equation}
\boxed{
\Delta(x\mid R)
=
\underbrace{
\sum_{i\in C\setminus(R\cup\{x\})}
\lambda_i\,
d_{A_{ix}}\!\bigl(S_i(R)\bigr)
}_{\text{support delivered to the remaining clients}}
-
\underbrace{
\lambda_x\phi\!\bigl(S_x(R)\bigr)
}_{\text{value forfeited by revealing }x}
}.
\label{eq:gm-reveal-marginal}
\end{equation}

\medskip
\noindent
\emph{Step 3: Comparison of the two marginals.}

Now fix
\[
R\subseteq R'\subseteq C,
\qquad
x\notin R'.
\]
Define
\[
I
:=
C\setminus(R'\cup\{x\})
\]
to be the tokens that remain supervised under both reveal sets, and define
\[
D
:=
R'\setminus R
\]
to be the tokens that are supervised under $R$ but have already been revealed
under $R'$.

Because $x\notin R'$, the client set under the smaller reveal set decomposes as
the disjoint union
\[
C\setminus(R\cup\{x\})
=
I\mathbin{\dot\cup}D.
\]
Using the marginal formula derived above, we have
\[
\begin{aligned}
\Delta(x\mid R)
={}&
\sum_{i\in I}
\lambda_i d_{A_{ix}}\!\bigl(S_i(R)\bigr)
+
\sum_{i\in D}
\lambda_i d_{A_{ix}}\!\bigl(S_i(R)\bigr)
-
\lambda_x\phi\!\bigl(S_x(R)\bigr),
\end{aligned}
\]
whereas
\[
\Delta(x\mid R')
=
\sum_{i\in I}
\lambda_i d_{A_{ix}}\!\bigl(S_i(R')\bigr)
-
\lambda_x\phi\!\bigl(S_x(R')\bigr).
\]
Subtracting the second expression from the first gives
\[
\begin{aligned}
\Delta(x\mid R)-\Delta(x\mid R')
={}&
\sum_{i\in I}
\lambda_i
\left[
d_{A_{ix}}\!\bigl(S_i(R)\bigr)
-
d_{A_{ix}}\!\bigl(S_i(R')\bigr)
\right] \\
&+
\sum_{i\in D}
\lambda_i d_{A_{ix}}\!\bigl(S_i(R)\bigr) \\
&+
\lambda_x
\left[
\phi\!\bigl(S_x(R')\bigr)
-
\phi\!\bigl(S_x(R)\bigr)
\right].
\end{aligned}
\]
Equivalently, restoring the definitions of $I$ and $D$,
\[
\begin{aligned}
\Delta(x\mid R)-\Delta(x\mid R')
={}&
\sum_{i\in C\setminus(R'\cup\{x\})}
\lambda_i
\left[
d_{A_{ix}}\!\bigl(S_i(R)\bigr)
-
d_{A_{ix}}\!\bigl(S_i(R')\bigr)
\right] \\
&+
\sum_{i\in R'\setminus R}
\lambda_i d_{A_{ix}}\!\bigl(S_i(R)\bigr) \\
&+
\lambda_x
\left[
\phi\!\bigl(S_x(R')\bigr)
-
\phi\!\bigl(S_x(R)\bigr)
\right].
\end{aligned}
\]

We now show that each of these three terms is nonnegative.

First, since $R\subseteq R'$ and $A_{ij}\geq 0$,
\[
\begin{aligned}
S_i(R')
&=
\sum_{j\in R'}A_{ij}
&=
\sum_{j\in R}A_{ij}
+
\sum_{j\in R'\setminus R}A_{ij}
&\geq
\sum_{j\in R}A_{ij}
=
S_i(R).
\end{aligned}
\]
Since $d_{A_{ix}}$ is nonincreasing,
\[
d_{A_{ix}}\!\bigl(S_i(R)\bigr)
\geq
d_{A_{ix}}\!\bigl(S_i(R')\bigr).
\]
Together with $\lambda_i\geq 0$, this proves that every summand in the first
line is nonnegative:
\[
\lambda_i
\left[
d_{A_{ix}}\!\bigl(S_i(R)\bigr)
-
d_{A_{ix}}\!\bigl(S_i(R')\bigr)
\right]
\geq 0.
\]

Second, monotonicity of $\phi$ and $A_{ix}\geq 0$ imply
\[
d_{A_{ix}}\!\bigl(S_i(R)\bigr)
=
\phi\!\bigl(S_i(R)+A_{ix}\bigr)
-
\phi\!\bigl(S_i(R)\bigr)
\geq 0.
\]
Therefore,
\[
\sum_{i\in R'\setminus R}
\lambda_i d_{A_{ix}}\!\bigl(S_i(R)\bigr)
\geq 0.
\]
This term accounts for clients that can benefit from revealing $x$ under the
smaller set $R$ but no longer contribute under $R'$ because they have already
been revealed.

Third, the same support-ordering argument gives
\[
S_x(R')
\geq
S_x(R).
\]
Since $\phi$ is nondecreasing,
\[
\phi\!\bigl(S_x(R')\bigr)
-
\phi\!\bigl(S_x(R)\bigr)
\geq 0.
\]
Because $\lambda_x\geq 0$,
\[
\lambda_x
\left[
\phi\!\bigl(S_x(R')\bigr)
-
\phi\!\bigl(S_x(R)\bigr)
\right]
\geq 0.
\]
This expresses the fact that revealing $x$ forfeits at least as much client
value under the larger reveal set, where $x$ has already accumulated more
support.

All three terms in
\[
\Delta(x\mid R)-\Delta(x\mid R')
\]
are therefore nonnegative. Consequently,
\[
\Delta(x\mid R)-\Delta(x\mid R') \geq 0 \quad\Longrightarrow\quad \Delta(x\mid R)\geq\Delta(x\mid R').
\]
Therefore $F$ is
submodular.

Finally, assuming $\phi(0)= 0$, and since $\phi$ is nondecreasing,
\[
\phi(s)\geq\phi(0)=0
\qquad\text{for every }s\geq0,
\]
and therefore $F(R)\geq0$ for every $R\subseteq C$.
\end{proof}
For our support function, $F$ is non-monotone whenever there exist
$i\neq j$ with $\lambda_i A_{ij}>0$. Since $\phi(s)>0$ for $s>0$,
\[
F(\{j\})\geq \lambda_i\phi(A_{ij})>0,
\]
whereas $\phi(0)=0$ gives $F(\emptyset)=0$, and $F(C)=0$ because no supervised
clients remain. Thus $F$ increases above zero at an intermediate reveal set and
must decrease again before the full set is revealed.

Equation~\ref{eq:gm-reveal-marginal} gives the gain used by the selector.
It also makes the cost of revealing a target explicit.

\section{Additional Results}
\label{app:additional-results}

\subsection{Fixed task-specific generation lengths}
\label{app:fixed-task-lengths}
Table~\ref{tab:fixed-task-lengths}
reports accuracy with generation lengths of 512 tokens for GSM8K and MATH-500
and 128 tokens for Countdown and Sudoku, shared across methods. Checkpoints
are selected by validation loss. These comparisons complement the
best-over-length results in Tables~\ref{tab:main} and~\ref{tab:12k}.

\begin{table}[!htbp]
\centering\small
\setlength{\tabcolsep}{6pt}
\renewcommand{\arraystretch}{0.96}
\caption{\textbf{Fixed task-specific generation lengths.} Accuracy (\%) at GSM8K/MATH-500/Countdown/Sudoku lengths of 512/512/128/128. Bold marks the highest displayed mean in each column within a panel.}
\label{tab:fixed-task-lengths}
\begin{tabular}{lcccc>{\columncolor{panelrow}}c}
\toprule
\multicolumn{6}{c}{\textbf{LLaDA-8B / s1K}}\\
\midrule
Method & GSM8K & MATH & Countdown & Sudoku & Average \\
\midrule
GIFT & \gmResultSD{81.0}{0.1} & \gmResultSD{37.3}{0.4} & \gmResultSD{21.5}{2.8} & \gmResultSD{15.8}{0.4} & \gmResultSD{38.9}{0.5} \\
CART & \gmResultSD{80.4}{1.0} & \gmResultSD{35.6}{2.0} & \gmResultSD{23.8}{1.7} & \gmResultSD{14.4}{1.6} & \gmResultSD{38.6}{0.6} \\
AGDO & \gmResultSD{78.3}{0.8} & \gmResultSD{37.2}{0.7} & \gmResultSD{15.5}{1.3} & \gmResultSD{12.4}{0.6} & \gmResultSD{35.8}{0.5} \\
LIFT & \gmResultSD{79.1}{0.8} & \gmResultSD{35.9}{1.0} & \gmResultSD{24.9}{3.0} & \gmResultSD{13.5}{1.0} & \gmResultSD{38.4}{0.7} \\
\midrule
\rowcolor{goldrow}
GoldiMask-D & \gmResultSD{80.4}{1.2} & \gmResultSD{\textbf{39.5}}{0.8} & \gmResultSD{21.9}{2.0} & \gmResultSD{16.3}{1.0} & \gmResultSD{\textbf{39.5}}{0.8} \\
\rowcolor{goldrow}
GoldiMask-R & \gmResultSD{80.4}{0.2} & \gmResultSD{38.3}{1.0} & \gmResultSD{\textbf{25.4}}{0.0} & \gmResultSD{13.4}{1.5} & \gmResultSD{39.4}{0.6} \\
\rowcolor{goldrow}
GoldiMask-T & \gmResultSD{\textbf{81.5}}{0.4} & \gmResultSD{38.7}{1.3} & \gmResultSD{21.3}{3.0} & \gmResultSD{15.1}{0.1} & \gmResultSD{39.2}{0.3} \\
\rowcolor{goldrow}
GoldiMask-P & \gmResultSD{81.0}{1.3} & \gmResultSD{36.5}{0.7} & \gmResultSD{22.5}{2.5} & \gmResultSD{\textbf{17.4}}{1.3} & \gmResultSD{39.4}{0.8} \\
\midrule
\multicolumn{6}{c}{\textbf{LLaDA-1.5 / s1K}}\\
\midrule
GIFT & \gmResultSD{80.4}{1.3} & \gmResultSD{35.9}{0.1} & \gmResultSD{17.4}{3.0} & \gmResultSD{19.0}{0.4} & \gmResultSD{38.2}{1.0} \\
CART & \gmResultSD{81.1}{1.1} & \gmResultSD{37.6}{1.8} & \gmResultSD{20.0}{2.3} & \gmResultSD{15.6}{0.6} & \gmResultSD{38.6}{0.3} \\
AGDO & \gmResultSD{79.8}{0.8} & \gmResultSD{38.8}{0.5} & \gmResultSD{20.7}{1.2} & \gmResultSD{14.3}{0.5} & \gmResultSD{38.4}{0.4} \\
LIFT & \gmResultSD{80.3}{1.1} & \gmResultSD{37.5}{1.9} & \gmResultSD{21.0}{3.7} & \gmResultSD{16.1}{0.8} & \gmResultSD{38.7}{1.1} \\
\midrule
\rowcolor{goldrow}
GoldiMask-D & \gmResultSD{81.8}{0.8} & \gmResultSD{37.6}{0.2} & \gmResultSD{20.0}{4.4} & \gmResultSD{\textbf{21.6}}{1.8} & \gmResultSD{40.3}{0.8} \\
\rowcolor{goldrow}
GoldiMask-R & \gmResultSD{81.3}{1.0} & \gmResultSD{40.5}{2.7} & \gmResultSD{22.9}{3.0} & \gmResultSD{17.9}{0.9} & \gmResultSD{40.7}{1.5} \\
\rowcolor{goldrow}
GoldiMask-T & \gmResultSD{\textbf{82.3}}{1.4} & \gmResultSD{40.5}{1.0} & \gmResultSD{\textbf{23.6}}{1.9} & \gmResultSD{17.6}{0.7} & \gmResultSD{\textbf{41.0}}{0.4} \\
\rowcolor{goldrow}
GoldiMask-P & \gmResultSD{81.7}{0.6} & \gmResultSD{\textbf{40.9}}{1.6} & \gmResultSD{22.1}{4.7} & \gmResultSD{16.7}{0.7} & \gmResultSD{40.4}{1.5} \\
\midrule
\multicolumn{6}{c}{\textbf{Dream-7B / s1K}}\\
\midrule
GIFT & \gmResultSD{77.3}{0.9} & \gmResultSD{40.2}{0.8} & \gmResultSD{23.8}{1.1} & \gmResultSD{22.7}{1.4} & \gmResultSD{41.0}{0.2} \\
CART & \gmResultSD{77.8}{0.2} & \gmResultSD{38.9}{0.6} & \gmResultSD{22.3}{0.8} & \gmResultSD{17.2}{2.7} & \gmResultSD{39.1}{0.8} \\
AGDO & \gmResultSD{77.6}{0.4} & \gmResultSD{40.3}{0.8} & \gmResultSD{29.1}{1.7} & \gmResultSD{10.8}{2.1} & \gmResultSD{39.4}{1.0} \\
LIFT & \gmResultSD{78.4}{0.1} & \gmResultSD{42.1}{0.7} & \gmResultSD{28.7}{0.8} & \gmResultSD{23.3}{2.5} & \gmResultSD{43.1}{0.7} \\
\midrule
\rowcolor{goldrow}
GoldiMask-D & \gmResultSD{77.9}{0.9} & \gmResultSD{42.8}{0.4} & \gmResultSD{35.9}{2.2} & \gmResultSD{20.3}{1.6} & \gmResultSD{44.2}{0.9} \\
\rowcolor{goldrow}
GoldiMask-R & \gmResultSD{\textbf{78.7}}{0.5} & \gmResultSD{\textbf{44.3}}{0.1} & \gmResultSD{30.1}{3.6} & \gmResultSD{\textbf{24.3}}{0.2} & \gmResultSD{44.4}{1.0} \\
\rowcolor{goldrow}
GoldiMask-T & \gmResultSD{78.3}{0.5} & \gmResultSD{43.5}{0.1} & \gmResultSD{\textbf{41.8}}{1.7} & \gmResultSD{24.2}{0.1} & \gmResultSD{\textbf{47.0}}{0.5} \\
\rowcolor{goldrow}
GoldiMask-P & \gmResultSD{\textbf{78.7}}{1.4} & \gmResultSD{44.0}{0.8} & \gmResultSD{40.8}{6.4} & \gmResultSD{21.1}{0.3} & \gmResultSD{46.2}{1.7} \\
\midrule
\multicolumn{6}{c}{\textbf{LLaDA-8B / LIFT-SFT-12K}}\\
\midrule
GIFT & \gmResultSD{79.8}{0.3} & \gmResultSD{36.1}{0.1} & \gmResultSD{23.6}{4.7} & \gmResultSD{8.2}{2.1} & \gmResultSD{36.9}{0.8} \\
CART & \gmResultSD{80.2}{0.9} & \gmResultSD{36.4}{2.8} & \gmResultSD{18.2}{3.6} & \gmResultSD{15.6}{1.2} & \gmResultSD{37.6}{1.7} \\
AGDO & \gmResultSD{\textbf{82.3}}{0.5} & \gmResultSD{34.7}{0.3} & \gmResultSD{17.8}{0.2} & \gmResultSD{10.3}{0.9} & \gmResultSD{36.3}{0.1} \\
LIFT & \gmResultSD{81.3}{0.7} & \gmResultSD{36.6}{0.5} & \gmResultSD{19.3}{3.5} & \gmResultSD{11.3}{0.3} & \gmResultSD{37.1}{1.0} \\
\midrule
\rowcolor{goldrow}
GoldiMask-D & \gmResultSD{82.2}{0.4} & \gmResultSD{\textbf{40.0}}{2.0} & \gmResultSD{23.0}{2.8} & \gmResultSD{\textbf{16.5}}{0.2} & \gmResultSD{40.4}{0.2} \\
\rowcolor{goldrow}
GoldiMask-R & \gmResultSD{82.0}{0.6} & \gmResultSD{36.5}{0.7} & \gmResultSD{17.8}{1.9} & \gmResultSD{14.4}{0.1} & \gmResultSD{37.7}{0.9} \\
\rowcolor{goldrow}
GoldiMask-T & \gmResultSD{81.7}{0.3} & \gmResultSD{39.0}{2.0} & \gmResultSD{\textbf{29.9}}{3.0} & \gmResultSD{14.0}{0.7} & \gmResultSD{\textbf{41.2}}{0.6} \\
\rowcolor{goldrow}
GoldiMask-P & \gmResultSD{81.0}{0.4} & \gmResultSD{37.1}{0.1} & \gmResultSD{25.4}{12.2} & \gmResultSD{14.1}{1.4} & \gmResultSD{39.4}{3.2} \\
\bottomrule
\end{tabular}
\end{table}

\subsection{AIME24 performance across seeds}
\label{app:aime24}

We additionally evaluate models fine-tuned on s1K on AIME24 using pass@16,
reporting means and sample standard deviations across three seeds in
Table~\ref{tab:aime24}. With only 30 problems, one additional solved problem
changes the score by approximately 3.3 percentage points. Several methods
show substantial variation across seeds, and the strongest method differs
across backbones. GoldiMask-D achieves the highest mean on Dream-7B, matches
LIFT on LLaDA-1.5, and slightly trails CART and LIFT on LLaDA-8B.

% Source: author's final corrected AIME24 scores, supplied 2026-09-22.
% Seed order: 42, 43, 44. Summaries use exact percentages from the integer
% solved counts out of 30 corresponding to the rounded scores below.
% Means and sample SDs (denominator n-1) are rounded only at the end.
% Method        LLaDA-8B         LLaDA-1.5        Dream-7B
% Vanilla       6.7,6.7,6.7      6.7,10.0,6.7     3.3,6.7,10.0
% GIFT          13.3,6.7,6.7     10.0,6.7,10.0    3.3,13.3,10.0
% CART          16.7,10.0,10.0   10.0,20.0,13.3   16.7,10.0,6.7
% LIFT          13.3,6.7,16.7    10.0,13.3,13.3   13.3,6.7,6.7
% GoldiMask-D   10.0,6.7,13.3    13.3,10.0,13.3   20.0,6.7,13.3
% AUTHOR CHECK: checkpoint rule for the supplied single-valued scores.
\begin{table}[!htbp]
\centering
\small
\setlength{\tabcolsep}{12pt}
\renewcommand{\arraystretch}{1.15}
\caption{\textbf{AIME24 pass@16 after fine-tuning on s1K.}
Values are mean accuracy (\%) $\pm$ sample standard deviation across three
seeds. Bold marks the highest mean in each column.}
\label{tab:aime24}
\begin{tabular}{@{}lccc@{}}
\toprule
\textbf{Method} & \textbf{LLaDA-8B} & \textbf{LLaDA-1.5} & \textbf{Dream-7B} \\
\midrule
Vanilla & \gmResultSD{6.7}{0.0} & \gmResultSD{7.8}{1.9} & \gmResultSD{6.7}{3.3} \\
GIFT & \gmResultSD{8.9}{3.8} & \gmResultSD{8.9}{1.9} & \gmResultSD{8.9}{5.1} \\
CART & \gmResultSD{\textbf{12.2}}{3.8} & \gmResultSD{\textbf{14.4}}{5.1} & \gmResultSD{11.1}{5.1} \\
LIFT & \gmResultSD{\textbf{12.2}}{5.1} & \gmResultSD{12.2}{1.9} & \gmResultSD{8.9}{3.8} \\
\midrule
\rowcolor{goldrow}
\textbf{GoldiMask-D} & \gmResultSD{10.0}{3.3} & \gmResultSD{12.2}{1.9} & \gmResultSD{\textbf{13.3}}{6.7} \\
\bottomrule
\end{tabular}
\end{table}

\section{Parallel-Decoding Behavior}
\label{app:parallel-decoding}

This section examines the parallel-decoding behavior observed after GoldiMask
fine-tuning. We describe the decoding process, evaluation setup, and results. GoldiMask itself operates during training; the analysis here
concerns the behavior of the resulting model at inference time.

\subsection{Decoding process}
\label{app:pd-description}

A masked diffusion decoder repeatedly predicts unresolved response positions
from a partially completed sequence. Several positions can be filled in the
same iteration because their predictions use the same current context. In a
confidence-threshold decoder, a position is eligible for commitment when the
probability of its most likely token reaches a threshold $\tau$. Committed
tokens then provide context for subsequent predictions. The choice of
threshold trades the size of the proposed batch against the reliability of
its predictions; any fallback or revision rule is also part of the decoder.

Confidence is observable during generation, but correctness is not. In
reference-based diagnostics we can additionally count predictions that both
match the reference token and satisfy the threshold. We call this quantity
\emph{oracle-safe parallelism}. It is a diagnostic of token recovery at a
fixed state, distinct from task-level correctness of freely generated text.

\subsection{Evaluation setup}
\label{app:pd-setup}

\paragraph{Models, tasks, and decoder.}
We use the same backbones and training corpora as Section~\ref{sec:main-results},
with the checkpoints used in the original parallel-decoding evaluation:
LLaDA-8B-Instruct trained on LIFT-SFT-12K, and LLaDA-1.5 and Dream-7B trained
on s1K, for GoldiMask-D, -R and -P and for the retrained baselines LIFT, CART,
GIFT and AGDO. Benchmarks are GSM8K,
MATH-500, Countdown and Sudoku, with maximum generation length $L=512$ for
MATH-500, $L=256$ for GSM8K, and
$L=128$ for Countdown and Sudoku. Following
dParallel~\citep{chen2026dparallel}, we use greedy confidence-threshold decoding
in 32-token blocks processed from left to right. Within each block, we
permanently commit masked-token predictions whose confidence reaches $\tau$,
falling back to the most confident prediction if none qualifies. All methods
decode the full prescribed length without EOS-based early stopping. We sweep
$\tau\in\{0.70,0.80,0.90,0.95\}$ and report, per setting and threshold, task
accuracy and the mean number of decoding iterations per example.

\paragraph{End-to-end decoding.}
The reported measurements are end-to-end: matched prompts and decoder settings,
with output quality reported together with total decoding iterations. Early
commitments change later context, so these trajectories are not comparable
state by state. Any wall-clock claim would additionally require matched
hardware and timing; a reduction in iteration count is not a measured latency
speedup.

\paragraph{Shared-state diagnostics.}
A complementary protocol holds partially completed reference sequences fixed
across models and measures correct eligible predictions, threshold precision,
and changes in correct eligibility. This separates prediction behavior from
visited states; these diagnostics are not reported below.

\subsection{Results}
\label{app:pd-results}

Figures~\ref{fig:pd-main} and~\ref{fig:pd-appendix} cover GSM8K/MATH-500 and
Countdown/Sudoku, respectively. Comparisons use the most accurate baseline
at each threshold.

\paragraph{Reasoning benchmarks.}
On GSM8K and MATH-500, all three variants on all three backbones reduce steps
by 15--21 percent at $\tau=0.95$, with accuracy differences of $-2.2$ to
$+1.4$ points. Step reductions span 12--32 percent across thresholds and are
largest at loose thresholds. On Dream-7B, however, accuracy trails the best
baseline by several points at $\tau=0.70$ and $0.80$, so the favorable
trade-off holds only at tight thresholds.

\paragraph{Countdown and Sudoku.}
These puzzles test whether parallel-decoding gains extend to globally
constrained outputs. Sudoku requires joint row, column, and subgrid
consistency; Countdown requires an expression using the available numbers
to reach a target. Token-wise confidence does not ensure joint validity,
and permanent commitments prevent later correction. This may explain the
more variable accuracy--iteration trade-offs, but our experiments neither
isolate this mechanism nor establish that these tasks are inherently
unsuitable for parallel decoding.

At $\tau=0.95$, step reductions range from $-8$ to $+16$ percent and accuracy
differences from $-11$ to $+10$ points. On Countdown, GoldiMask-D and -P
substantially outperform the best baseline on LLaDA-8B, while all variants
trail CART on LLaDA-1.5 and LIFT on Dream-7B. Sudoku accuracy differences
remain within three points in every setting.

\begin{figure}[htb!]
\centering
\includegraphics[width=0.9\linewidth]{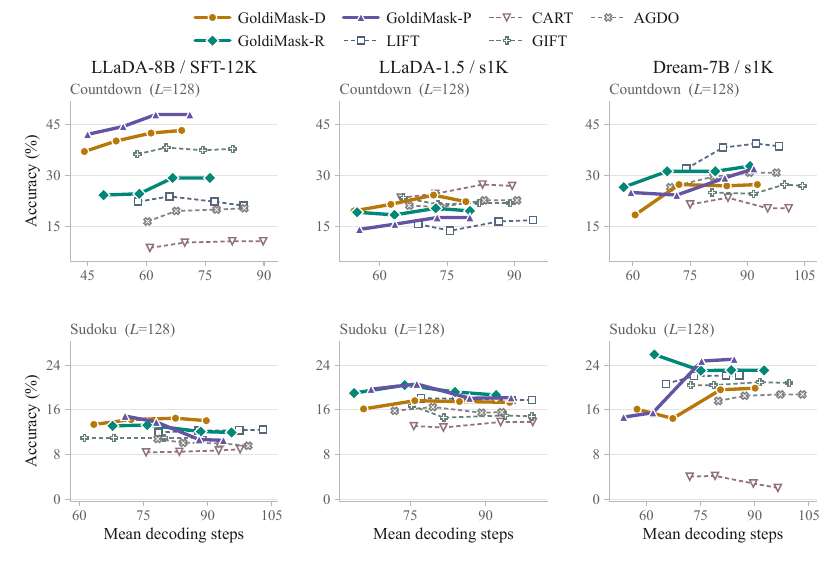}
\caption{\textbf{Accuracy versus decoding cost on Countdown and Sudoku.} Same
conventions as Figure~\ref{fig:pd-main}: three backbones, thresholds
$\tau\in\{0.70,0.80,0.90,0.95\}$ from left to right, GoldiMask-D, -R and -P
solid against the SFT baselines dashed. Countdown uses $L=128$ on all
backbones.}
\label{fig:pd-appendix}
\end{figure}

\section{Evaluation Protocol and Compute Resources}
\label{sec:computeresources}

\paragraph{Checkpoint and generation-length selection.}
\label{app:checkpoint-selection}
For all experiments, we select each seed's checkpoint solely by the
lowest loss on a held-out validation split, following LIFT. For reasoning comparisons other than Table~\ref{tab:fixed-task-lengths},
generation length is selected separately for each benchmark from
$\{128,256,512\}$, with the same procedure across methods.

\paragraph{Training setup and compute resources.}
We follow the training and evaluation protocol of
LIFT~\citep{parashar2026learnabilityinformed-lift}. Post-training experiments
were conducted using four NVIDIA B200 GPUs. We reimplement
AGDO~\citep{deng2026beyond-agdo} in our shared harness with matched training
budgets and hyperparameters. Each masked candidate is scored by its attention
mass on visible context, and the highest-scoring candidates are revealed.
The remaining supervised tokens receive weights $1+\gamma I_i$, where $I_i$
is AGDO's token importance score and $\gamma=100$. We normalize these weights
to have mean one for comparability. This implementation does not incorporate
GoldiMask's selection objective or weighting mechanism.

\paragraph{Training time.}
Table~\ref{tab:training-time} reports wall-clock training time for
LLaDA-8B-Instruct relative to Vanilla SFT at matched optimizer steps and
settings on four B200 GPUs. The measurements include all training
operations and exclude model loading and evaluation. The reported relative
times are averaged across s1K and LIFT-SFT-12K.

% Ratios preserved from rebuttal/single_pass_training.tex. Hardware confirmed
% as four B200 GPUs by the author on 2026-09-24.
\begin{table}[!htbp]
\centering\small
\setlength{\tabcolsep}{10pt}
\caption{\textbf{Relative training time.} LLaDA-8B-Instruct on four B200
GPUs. Values are averaged across s1K and LIFT-SFT-12K, with Vanilla SFT
as the $1.00\times$ reference.}
\label{tab:training-time}
\begin{tabular}{@{}lc@{}}
\toprule
\textbf{Method} & \textbf{Training time relative to Vanilla SFT} \\
\midrule
Vanilla SFT & $1.00\times$ \\
CART & $1.01\times$ \\
GIFT & $1.17\times$ \\
AGDO & $1.17\times$ \\
LIFT & $1.24\times$ \\
\midrule
\rowcolor{goldrow}
GoldiMask-D & $1.21\times$ \\
\rowcolor{goldrow}
GoldiMask-R & $1.23\times$ \\
\bottomrule
\end{tabular}
\end{table}

\section{Loss Weighting}
\label{app:gm-weighting}

After selecting context, GoldiMask converts the measured utilities into loss
weights. Section~\ref{sec:component-results} compares alternative weighting
schemes on LLaDA-8B trained on s1K; here we explain the allocation properties
of the chosen rule.
Water-filling solves a fixed-utility allocation problem
(Proposition~\ref{prop:waterfill}), and the cap bounds weight concentration
(Corollary~\ref{cor:ess}). These properties concern the allocation of
loss weight within a training example; they do not imply a fixed gradient
norm or optimal downstream learning.

\subsection{Properties of the loss weighting}
\label{app:gm-weight-properties}

%The same priority function is evaluated at post-reveal confidence in Equation~\ref{eq:utility}. Multiplying it by positive context gain combines evidence that a target responds with its remaining priority.
The utility in Equation~\ref{eq:utility} multiplies the positive log-gain of a target by the priority function $\lambda$ evaluated at its post-reveal
confidence. The product combines evidence that the target responds to the
selected context with its remaining priority.
%This is a designed allocation score; the local-update argument motivates its confidence factor, not the entire product with log-gain.

Each training example has a fixed total loss-weight mass $\sum_i w_i = n$ to divide among its supervised tokens. ``Proportionally fair'' is the standard name for one particular answer: an
allocation is proportionally fair when no reallocation can raise the shares of
some tokens by more, in utility-weighted proportional terms, than it lowers
those of the others. % CHANGED (aggregate, not pairwise, definition)
That criterion is exactly
what maximizing $\sum_i u_i \log w_i$ encodes, because the logarithm measures
each token's satisfaction in \emph{relative} terms, so halving any token's
weight costs the same amount whether it held a large share or a small one, and
the utilities $u_i$ set how loudly each token argues. The same objective defines
the Nash bargaining solution, which is why the two coincide below.

\begin{proposition}[Water-filling weights are proportionally fair]
\label{prop:waterfill}
Fix utilities $u_1, \dots, u_n \ge 0$, let $P = \{i : u_i > 0\}$ with
$n_+ = |P|$, a cap $c > 1$, and the budget polytope
$W = \{w \in \mathbb{R}^n : \sum_i w_i = n,\ 0 \le w_i \le c\}$ (nonempty since
$c > 1$). Consider the problem $\max_{w \in W} \sum_{i \in P} u_i \log w_i$. % CHANGED
\begin{enumerate}
\item[(a)] If $n_+ c > n$, it has the unique solution
$w_i^\star = \min\!\big(u_i/\nu^\star,\ c\big)$ for $i \in P$ and
$w_i^\star = 0$ for $i \notin P$, where $\nu^\star > 0$ is the unique root of
$\sum_{i \in P} \min(u_i/\nu, c) = n$. Equivalently, $w^\star$ is the
asymmetric Nash bargaining solution over the budget set with bargaining powers
$u_i$ and disagreement point $0$. % CHANGED (original statement, restricted to P)
\item[(b)] If $n_+c=n>0$, the unique weights are $c$ on $P$ and zero
elsewhere, although the water-filling level is not unique.
\item[(c)] If $0<n_+ c < n$, every solution has $w_i^\star = c$ for all
$i \in P$; the remaining budget $n - n_+ c$ may be spread over $P^{c}$
arbitrarily, and the implementation spreads it uniformly. % NEW
\item[(d)] If all utilities vanish, the objective is constant on $W$ and the
implementation convention is $w_i=1$ for every $i$.
\end{enumerate}
\end{proposition}

\begin{proof}
\emph{Reduction to $P$.} % NEW paragraph
The objective does not depend on $w_i$ for $i \notin P$, and moving mass from
such an $i$ to any uncapped $j \in P$ strictly increases it. Hence at any
maximizer either every token in $P$ is capped or $w_i = 0$ on $P^{c}$. If
$0<n_+ c < n$, capping all of $P$ is forced and leaves $n - n_+ c > 0$ to be
placed on $P^{c}$, where the objective is constant; this is (c). If
$n_+ c \ge n$, the budget can be met inside $P$, so $w_i = 0$ on $P^{c}$ and
the problem reduces to the coordinates in $P$, that is, to the polytope $W_P = \{w \in \mathbb{R}^{P} : \sum_{i \in P} w_i = n,\ 0 \le w_i \le c\}$. The rest of the proof treats
this case. If $P$ is empty, the objective is constant, and uniform weighting
is feasible; this gives (d).

\emph{Existence and uniqueness.} On $W_P \cap \{w > 0\}$ the objective
$\sum_{i \in P} u_i \log w_i$ is strictly concave, because each term is a positive % CHANGED
multiple of the strictly concave $\log$. As any coordinate $w_i \to 0$ the
objective tends to $-\infty$, so no maximizer touches the lower boundary, and
$W_P$ is compact and convex. The objective is upper semicontinuous on $W_P$ with
values in $[-\infty,\infty)$, so a maximum is attained, and by strict concavity
at exactly one point. % CHANGED

\emph{Shape of the solution.} At the maximizer, the Karush--Kuhn--Tucker
conditions hold: there are multipliers $\nu$ for the budget $\sum_i w_i = n$ and
$\mu_i \geq 0$ for the caps $w_i \leq c$ with
\[
\frac{u_i}{w_i} \;=\; \nu + \mu_i,
\qquad
\mu_i\,(w_i - c) = 0 .
\]
Here $\nu > 0$ in case (a): otherwise every $\mu_i > 0$, every cap binds, and
$\sum_{i \in P} w_i = n_+ c > n$. At equality, every positive-utility
coordinate must instead be capped, which gives the unique weights in (b).
Read the two cases off complementary slackness. If the cap is slack
($w_i < c$), then $\mu_i = 0$ and the stationarity condition gives
$w_i = u_i/\nu$. If the cap binds, then $w_i = c$, and stationarity requires
$\mu_i = u_i/c - \nu \geq 0$, i.e.\ this case occurs exactly when
$u_i/\nu \geq c$. Both cases are summarized by
\[
w_i \;=\; \min\!\big(u_i/\nu,\; c\big).
\]

\emph{The level $\nu^\star$ exists and is unique.} Define
$g(\nu) = \sum_{i \in P} \min(u_i/\nu, c)$, the total weight assigned at level $\nu$. % CHANGED
Each summand is continuous and nonincreasing in $\nu$, and strictly decreasing
wherever $u_i/\nu < c$; moreover
\[
\lim_{\nu \to 0^+} g(\nu) = n_+ c, % CHANGED
\qquad
\lim_{\nu \to \infty} g(\nu) = 0 .
\]
Since $n < n_+ c$ in case (a), the intermediate value theorem gives a root of
$g(\nu) = n$, and strict monotonicity below $n_+ c$ makes it unique. At
$n=n_+c$, every $0<\nu\leq\min_{i\in P}u_i/c$ gives the same all-capped
weights, so the weights are unique but the level is not. In case (a), the root is the
$\nu^\star$ found by bisection in the implementation.

\emph{Nash bargaining.} The asymmetric Nash bargaining solution with
bargaining powers $u_i$ and disagreement point $0$ maximizes
$\prod_{i \in P} w_i^{\,u_i}$ over the feasible set, and taking logarithms turns this % CHANGED
product into $\sum_{i \in P} u_i \log w_i$, the objective above. The two solutions % CHANGED
therefore coincide.
\end{proof}

\begin{corollary}[A bound on the effective number of supervised tokens]
\label{cor:ess}
For any $w \in W$, the effective sample size $n_{\mathrm{eff}} = (\sum_i w_i)^2 / \sum_i w_i^2$ satisfies $n_{\mathrm{eff}} \ge n / c$.
\end{corollary}

\begin{proof}
$\sum w_i^2 \le c \sum w_i = cn$, so $n_{\mathrm{eff}} \ge n^2 / (cn) = n/c$.
\end{proof}

This is an effective-size bound on the loss weights. It does not count
independent gradients or imply a fixed gradient norm. The cap limits weight
concentration, while water-filling preserves total loss-weight mass.

Proposition~\ref{prop:waterfill} shows that whenever the positive utilities can absorb the budget (i.e., $n_+c \ge n$), the water-filling weights are the \emph{unique} maximizer of $\sum_{i \in P} u_i \log w_i$ for a fixed total loss-weight mass under a per-token cap. Compared to normalize-then-clip, $w_i = \min(n\,u_i/\sum_j u_j,\ c)$,
water-filling has the same form but a different level. Normalize-then-clip
keeps $\nu' = \sum_j u_j/n$, while water-filling adjusts it to $\nu^\star$ so
that the weights still sum to $n$. Whenever an uncapped normalized weight exceeds the cap, the
normalize-then-clip weights sum to less than $n$: the clipped mass
$\sum_i \big(n\,u_i/\sum_j u_j - c\big)_+$ is lost, not redistributed. It
therefore solves the same problem for a smaller mass.

\section{Supervision Priority and Contextual Support}
\label{app:gm-functions}
\label{sec:measurements}

Our approach uses the supervision priority $\lambda(p)=p(1-p)^3$ and the
contextual support response $\phi(s)=s/(s+0.05)$ to allocate supervision and
revealed context. In what follows, we present theoretical and empirical support
for these choices. Section~\ref{app:gm-lambda-theory} shows that weighting the
local probability response by remaining prediction headroom yields our priority
function under a simplified training update, giving it an explicit mathematical
motivation. Section~\ref{app:gm-support-theory} establishes that $\phi$ rewards
additional contextual support with diminishing returns, reducing the marginal
reward for further supporting an already well-supported target.
Section~\ref{app:gm-training-study} shows that sampling supervision according to
$\lambda$ improves both held-out denoising loss and average correct-token
probability over uniform random selection on both tested backbones.
Section~\ref{app:gm-lambda-evidence} finds that additional context produces the
largest absolute probability gains at intermediate initial confidence across
the tested models and corpora, supporting confidence-dependent context
allocation. Section~\ref{app:gm-support} finds that a saturating response fits
contextual support measurements better than a linear response in high-noise
settings, with fitted support scales consistent with our default $s_0=0.05$.
This section does not establish that these functions are optimal, but explains
why we chose them. Additional ablation studies in Section~\ref{sec:component-results}
provide further empirical support for these choices.

\subsection{Motivation for the supervision-priority function}
\label{app:gm-lambda-theory}

GoldiMask allocates a limited supervision budget among masked targets. We seek
a priority that accounts for both a target's response to additional supervision
and its remaining room for improvement. To motivate this choice, we consider
how a small cross-entropy update changes the probability assigned to the correct
token.

Under a simple local model, this probability increase scales as $p(1-p)^2$.
We additionally weight the response by the remaining probability headroom,
$1-p$, to place greater emphasis on predictions that still need improvement.
This gives $\lambda(p)=p(1-p)^3$. The headroom factor shifts the maximum
from $p=1/3$ for the response alone to $p=1/4$ for our priority. The following
proposition formalizes this construction.

\begin{proposition}[Local response weighted by remaining headroom]
\label{prop:gm-local-priority}
Let $p=\sigma(m)\in(0,1)$ be the gold-token probability expressed through a
gold-versus-rest margin $m$. Apply one gradient-descent step of size $\eta$
to $\ell(m)=-\log\sigma(m)$, treating $m$ as the trainable scalar. If $p^+$
is the updated probability, then, as $\eta\to0$,
\begin{equation}
p^+-p=\eta p(1-p)^2+O(\eta^2),\qquad
(1-p)(p^+-p)=\eta p(1-p)^3+O(\eta^2).
\label{eq:gm-headroom-priority}
\end{equation}
The confidence factor in the headroom-weighted response,
$\lambda(p)=p(1-p)^3$, has its unique maximum at $p=1/4$.
\end{proposition}
\begin{proof}
Since $\ell'(m)=-(1-p)$, the updated margin is $m^+=m+\eta(1-p)$.
Using $\sigma'(m)=p(1-p)$, Taylor expansion gives
$\sigma(m^+)-\sigma(m)=\eta p(1-p)^2+O(\eta^2)$.
Multiplying by $1-p$ gives the second identity. Finally,
$\lambda'(p)=(1-p)^2(1-4p)$ is positive below $1/4$ and negative above it,
while $\lambda(0)=\lambda(1)=0$.
\end{proof}

This construction gives our priority a concrete interpretation as local
probability response weighted by remaining headroom. The additional headroom
factor expresses our allocation preference. In GoldiMask's reveal objective,
this priority is multiplied by contextual support, favoring reveal sets that
provide useful evidence to high-priority targets left masked.

The same confidence factor also appears beyond the scalar model. For a direct
cross-entropy update of all vocabulary logits, let $q$ be the softmax
probability vector, $y$ the correct token, and $q_y=p$. Then
\[
(1-p)\Delta p
=\eta\left(1+\sum_{j\ne y}r_j^2\right)p(1-p)^3+O(\eta^2),
\]
where $r_j=q_j/(1-p)$ describes the conditional probability distribution over
incorrect tokens. The additional factor lies between $1+1/(V-1)$ and $2$,
where $V$ is the vocabulary size. Updates to shared model parameters also
depend on parameter sensitivity and interactions among supervised targets.
These calculations motivate a simple confidence-based proxy;
Section~\ref{app:gm-training-study} examines its usefulness for allocating
supervision during actual fine-tuning.

\subsection{Motivation for the saturating support function}
\label{app:gm-support-theory}

After assigning supervision priorities, GoldiMask must assess how much support
a reveal set provides to the remaining targets. A linear support score gives
every additional unit of attention support the same value, regardless of how
much support a target already receives. We instead use a saturating response,
so that additional support remains beneficial but contributes less as support
accumulates. This encourages the reveal objective to account for targets that
still lack contextual support.

For fixed attention, the support received by target $i$ from the revealed set
$R$ is $S_i(R)=\sum_{j\in R}A_{ij}$. We map this support to a bounded score
using
\[
\phi(s)=\frac{s}{s+s_0},\qquad s_0>0.
\]
The parameter $s_0$ controls the support scale at which the response reaches
half its maximum. GoldiMask uses $s_0=0.05$.

\begin{proposition}[Diminishing returns from support]
\label{prop:gm-support-response}
For $s_0>0$, the response $\phi(s)=s/(s+s_0)$ is increasing and concave on
$s\geq0$, with
\[
\phi(0)=0,\qquad \phi(s_0)=\tfrac12,\qquad
\lim_{s\to\infty}\phi(s)=1.
\]
Consequently, for any fixed additional support $a\geq0$, the gain
$\phi(s+a)-\phi(s)$ is nonincreasing in $s$.
\end{proposition}
\begin{proof}
Direct differentiation gives
\[
\phi'(s)=\frac{s_0}{(s+s_0)^2}>0,\qquad
\phi''(s)=-\frac{2s_0}{(s+s_0)^3}<0.
\]
The endpoint identities follow by substitution and taking the limit. Moreover,
\[
\phi(s+a)-\phi(s)=\frac{a s_0}{(s+s_0)(s+a+s_0)},
\]
which is nonincreasing in $s$.
\end{proof}

Thus, for targets with equal supervision priority, the same additional support
earns a larger marginal reward when existing support is smaller. Combined with
$\lambda$, this lets the objective account for both target priority and
diminishing contextual benefit. Concavity also underpins the submodularity
result in Proposition~\ref{prop:submodular}. Section~\ref{app:gm-support}
examines saturation empirically and provides the measurements informing our
default support scale.

\subsection{Study 1: confidence-based supervision and subsequent learning}
\label{app:gm-training-study}

Section~\ref{app:gm-lambda-theory} motivates $\lambda(p)=p(1-p)^3$ through
local probability response and remaining headroom. We now ask whether using
this priority to choose supervised targets improves learning on unseen
examples. The experiment holds the number of supervised tokens fixed and
changes only how those tokens are selected.

\paragraph{A matched supervision budget.}
We train LLaDA-8B and LLaDA-1.5 on 600 examples from s1K, indexed
$[300,900)$, for 300 optimizer steps with two seeds per backbone. Within each
seed, all rules start from the same model state and use the same examples,
order, and corruptions. For each example, exactly 32 masked positions carry
the training loss. The remaining masked positions stay masked but do not
contribute directly to that example's loss. We compare uniform random
selection with sampling proportional to $p(1-p)^3$, $1-p$, or $p(1-p)$,
where $p$ is the current model's probability of the correct token before the
update. This isolates the supervision-priority component; the experiment does
not run GoldiMask's full reveal-selection and capped-weighting procedure.

All runs use LoRA with rank 128, scaling 256, and dropout 0.05 on the
query/key/value projections. Optimization uses AdamW with learning rate
$5\times10^{-6}$, weight decay 0.1, gradient clipping at 1.0, and global
batch size 8. The four reported rules give 16 training blocks across the two
backbones and two seeds.
% The full eight-rule experiment campaign comprised 32 blocks and approximately
% 55 A40 GPU-hours, including the band-only arms and all-token reference omitted here.

\paragraph{Measuring learning on unseen examples.}
At steps 0, 100, and 300, we evaluate 90 held-out examples using the same
fixed masks at masking rates $t\in\{0.25,0.5,0.75\}$. For a masked target,
negative log-likelihood (NLL) is $-\log p$, where $p$ is the probability of
its correct token. We report two changes from before to after training.
$\Delta\mathrm{NLL}$ is the change in average held-out denoising loss, so
negative values indicate improvement. $\Delta p$ is the average increase
in correct-token probability, evaluated on held-out targets whose initial
probability is below 0.9, so positive values indicate improvement. These
measurements compare model checkpoints on fixed inputs, rather than responses
to newly revealed context. NLL and probability gain summarize different
aspects of prediction quality and can rank supervision rules differently.

\begin{table}[htb!]
\centering
\caption{\textbf{Learning under a fixed supervision budget.} Changes after
300 training steps, with 32 supervised positions per training example.
$\Delta\mathrm{NLL}$ is the change in held-out negative log-likelihood;
$\Delta p$ is the mean correct-token probability change for held-out targets
with initial $p<0.9$. Arrows indicate the preferred direction. Bold values
are best within each column; shading identifies our priority.}
\label{tab:gm-training-study}
\begingroup
\small
\setlength{\tabcolsep}{14pt}
\renewcommand{\arraystretch}{1.18}
\begin{tabular}{lrrrr}
\toprule
& \multicolumn{2}{c}{\textbf{LLaDA-8B}} & \multicolumn{2}{c}{\textbf{LLaDA-1.5}}\\
\cmidrule(lr){2-3}\cmidrule(l){4-5}
Supervision rule
& $\Delta\mathrm{NLL}\,\downarrow$ & $\Delta p\,\uparrow$
& $\Delta\mathrm{NLL}\,\downarrow$ & $\Delta p\,\uparrow$\\
\midrule
Uniform random & $-0.013$ & $+0.013$ & $-0.015$ & $+0.011$\\
$1-p$ & $\mathbf{-0.017}$ & $-0.013$ & $\mathbf{-0.019}$ & $-0.014$\\
$p(1-p)$ & $-0.001$ & $\mathbf{+0.052}$ & $-0.003$ & $\mathbf{+0.051}$\\
\addlinespace[3pt]
\rowcolor{goldrow}
Our priority $p(1-p)^3$ & $-0.016$ & $+0.032$ & $-0.018$ & $+0.030$\\
\bottomrule
\end{tabular}
\endgroup
\end{table}

\paragraph{Held-out results.}
Our priority improves both held-out NLL and average correct-token probability
more than uniform random selection on both backbones
(Table~\ref{tab:gm-training-study}). The alternatives favor different outcomes.
The $1-p$ rule reduces NLL slightly more but lowers mean correct-token
probability; $p(1-p)$ raises mean probability more but achieves much smaller
NLL reductions. Our priority provides a balance between the two metrics.

\subsection{Study 2: context-response profiles across models and corpora}
\label{app:gm-lambda-evidence}

GoldiMask uses target confidence to guide the allocation of revealed context.
We therefore examine how the benefit of additional context varies with a
target's initial confidence. Using 100 examples and four paired masking draws
per example, we evaluate each frozen model under two nested masks. For targets
that remain masked in both inputs, we measure
\begin{equation}
\Delta p=p_2-p_1,
\label{eq:gm-context-score}
\end{equation}
where $p_1$ and $p_2$ are the correct-token probabilities before and after
revealing additional context. We group these gains by initial confidence $p_1$.

Across the tested models and corpora, absolute probability gains peak at
intermediate confidence. The LLaDA-8B measurements illustrate a broad
high-response region. Average gains are $0.46$ for initial confidence
$0.05$--$0.10$, $0.51$ for $0.10$--$0.20$, and $0.49$ for
$0.20$--$0.30$, falling to $0.01$ above $0.90$. Thus, although the fitted
maximum lies near $0.16$, the neighboring bin containing $0.25$ retains nearly
the same average gain.

Table~\ref{tab:measurement} summarizes fitted peak locations across all four
model--corpus settings. The confidence intervals quantify uncertainty in the
peak location. Held-out $R^2$ measures how well a curve fitted on half the
examples describes the response profile on the other half, with $1$ indicating
a perfect fit.

\begin{table}[htb!]
\centering\small
\caption{Study 2: absolute probability gain indexed by pre-reveal confidence.
Peak intervals use an example-level bootstrap. The power-family fit is trained
on half the examples and evaluated on the other half.}
\label{tab:measurement}
\begin{tabular}{lcc}
\toprule
Model / corpus & Fitted peak [95\% CI] & Held-out $R^2$\\
\midrule
LLaDA-8B / s1K & $0.159\ [0.148,0.170]$ & 0.958\\
LLaDA-1.5 / s1K & $0.157\ [0.147,0.168]$ & 0.952\\
Dream-7B / s1K & $0.164\ [0.154,0.177]$ & 0.968\\
LLaDA-8B / 12K & $0.163\ [0.144,0.179]$ & 0.973\\
\bottomrule
\end{tabular}
\end{table}

These measurements support confidence-dependent context allocation. The
substantial gains across neighboring intermediate-confidence bins are
consistent with GoldiMask's smooth priority, which emphasizes responsive
targets while retaining a preference for remaining headroom.

\subsection{Study 3: diminishing returns from contextual support}
\label{app:gm-support}
\label{app:gm-function-controls}

Study 2 examines which targets respond to context. Here we test whether
prediction responses support GoldiMask's saturating contextual support score
$\phi(s)=s/(s+0.05)$ rather than a linear response.

\paragraph{Comparing targets across contexts.}
On LLaDA-8B, we evaluate eight random reveal sets per configuration over
25{,}217 targets, holding target masks, noise level, and reveal budget fixed.
For each target and reveal set, we record attention mass on revealed tokens
and the correct token's log-probability. Target-specific intercepts account
for fixed differences in difficulty. For held-out evaluation, we center
responses and predictors within each target cell, fit one slope per noise bin
on even-indexed examples, and evaluate on odd-indexed examples. We swap the
halves and average the two scores.

\paragraph{Evidence for saturation.}
With $s_0=0.05$, saturation improves held-out within-target $R^2$ from
$0.033$ to $0.041$ and from $0.082$ to $0.098$ in masking intervals
$0.50$--$0.71$ and $0.71$--$0.84$, respectively
(Table~\ref{tab:gm-support-study}). These scores measure explained variation
across contexts after removing target-specific means. The advantage is smaller
above $t=0.84$ and absent below $t=0.25$, supporting diminishing returns
particularly in the two reported intervals.

\paragraph{Support scale.}
The response $\phi(s)=s/(s+s_0)$ reaches half its maximum at $s=s_0$.
Fitted scales span $0.05$--$0.2$ across noise bins, placing our default $0.05$
at the lower end of the observed range. For equally prioritized targets,
$\phi$ rewards additional support more when existing support is smaller.

% Source: the support-curve statistics in appendix/measurement_protocols.tex.
% These are reported summaries; the raw measurements are not in this repository.
% Held-out fit scores and split protocol confirmed by the author on 2026-09-22.
% Author confirmed on 2026-09-25 that the default and empirical fits use
% s_0=0.05. Fit-parameter labels corrected; reported R^2 values unchanged.
\begin{table}[htb!]
\centering\small
\caption{\textbf{Study 3: contextual support and diminishing returns.}
Support fits on LLaDA-8B. Arrows compare held-out within-target $R^2$
for the linear and saturating responses, with $s_0=0.05$ for the latter.}
\label{tab:gm-support-study}
\setlength{\tabcolsep}{8pt}
\renewcommand{\arraystretch}{1.12}
\begin{tabular}{lr}
\toprule
Measurement & Result\\
\midrule
Held-out $R^2$, masking interval $0.50$--$0.71$ & $0.033\rightarrow 0.041$\\
Held-out $R^2$, masking interval $0.71$--$0.84$ & $0.082\rightarrow 0.098$\\
Fitted half-saturation scale $s_0$ across noise bins & $0.05$--$0.2$\\
\bottomrule
\end{tabular}
\end{table}

\begin{figure}[htb!]
\centering
\begin{tikzpicture}
\begin{axis}[
width=0.47\linewidth,height=3.5cm,
xlabel={Gold-token probability $p$},ylabel={Normalized priority},
xmin=0,xmax=1,ymin=0,ymax=1.05,
tick label style={font=\scriptsize},label style={font=\small},
legend style={font=\scriptsize,draw=none,at={(0.98,0.98)},anchor=north east}]
\addplot[goldline,thick,domain=0:1,samples=150]{256/27*x*(1-x)^3};
\addlegendentry{$p(1-p)^3$}
\addplot[gray,dashed,domain=0:1,samples=100]{4*x*(1-x)};
\addlegendentry{$p(1-p)$}
\addplot[gray,dotted,domain=0:1]{1-x};
\addlegendentry{$1-p$}
\end{axis}
\end{tikzpicture}\hfill
\begin{tikzpicture}
\begin{axis}[
width=0.47\linewidth,height=3.5cm,
xlabel={Attention support $s$},ylabel={Support response},
xmin=0,xmax=0.6,ymin=0,ymax=1.05,
tick label style={font=\scriptsize},label style={font=\small},
legend style={font=\scriptsize,draw=none,at={(0.98,0.05)},anchor=south east}]
\addplot[goldline,thick,domain=0:0.6,samples=100]{x/(x+0.05)};
\addlegendentry{$s_0=0.05$}
\addplot[gray,dashed,domain=0:0.6,samples=100]{x/(x+0.1)};
\addlegendentry{$s_0=0.1$}
\addplot[gray,dotted,domain=0:0.6,samples=100]{x/(x+0.2)};
\addlegendentry{$s_0=0.2$}
\end{axis}
\end{tikzpicture}
\caption{\textbf{GoldiMask's functional forms} (analytical, not empirical).
Left: priorities normalized to unit maximum. Right: support responses for
$s_0=0.05$ (default), $0.1$, and $0.2$. Empirical fits appear in
Table~\ref{tab:gm-support-study}.}
\label{fig:curves}
\end{figure}
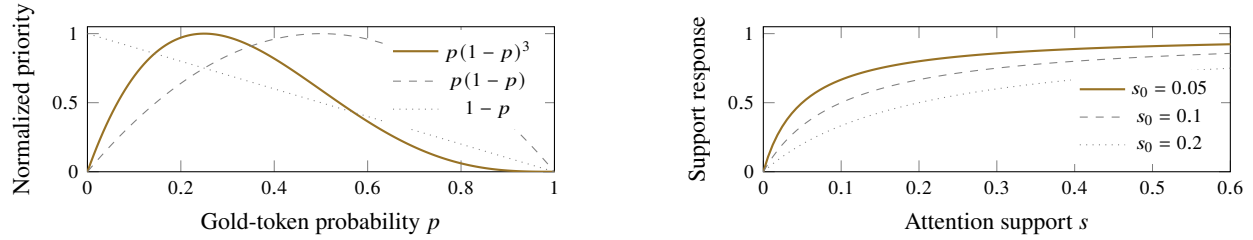

\end{document}